\documentclass[onefignum,onetabnum]{siamonline190516}

\usepackage{braket,amsfonts}

\usepackage{array}

\usepackage{pgfplots}

\newsiamthm{claim}{Claim}
\newsiamremark{remark}{Remark}
\newsiamremark{hypothesis}{Hypothesis}
\crefname{hypothesis}{Hypothesis}{Hypotheses}

\usepackage{algorithmic}

\usepackage{graphicx,epstopdf}

\Crefname{ALC@unique}{Line}{Lines}

\usepackage{amsopn}
\DeclareMathOperator{\Range}{Range}

\usepackage{xspace}
\usepackage{bold-extra}
\usepackage[most]{tcolorbox}

\colorlet{texcscolor}{blue!50!black}
\colorlet{texemcolor}{red!70!black}
\colorlet{texpreamble}{red!70!black}
\colorlet{codebackground}{black!25!white!25}
\usepackage{multirow}
\newcommand\bs{\symbol{'134}} 
\newcommand{\preamble}[2][\small]{\textcolor{texpreamble}{#1\texttt{#2 \emph{\% <- Preamble}}}}

\lstdefinestyle{siamlatex}{%
  style=tcblatex,
  texcsstyle=*\color{texcscolor},
  texcsstyle=[2]\color{texemcolor},
  keywordstyle=[2]\color{texemcolor},
  moretexcs={cref,Cref,maketitle,mathcal,text,headers,email,url},
}

\tcbset{%
  colframe=black!75!white!75,
  coltitle=white,
  colback=codebackground, 
  colbacklower=white, 
  fonttitle=\bfseries,
  arc=0pt,outer arc=0pt,
  top=1pt,bottom=1pt,left=1mm,right=1mm,middle=1mm,boxsep=1mm,
  leftrule=0.3mm,rightrule=0.3mm,toprule=0.3mm,bottomrule=0.3mm,
  listing options={style=siamlatex}
}

\newtcblisting[use counter=example]{example}[2][]{%
  title={Example~\thetcbcounter: #2},#1}

\newtcbinputlisting[use counter=example]{\examplefile}[3][]{%
  title={Example~\thetcbcounter: #2},listing file={#3},#1}

\DeclareTotalTCBox{\code}{ v O{} }
{ 
  fontupper=\ttfamily\color{black},
  nobeforeafter,
  tcbox raise base,
  colback=codebackground,colframe=white,
  top=0pt,bottom=0pt,left=0mm,right=0mm,
  leftrule=0pt,rightrule=0pt,toprule=0mm,bottomrule=0mm,
  boxsep=0.5mm,
  #2}{#1}

\patchcmd\newpage{\vfil}{}{}{}
\usepackage{enumitem}
\usepackage{subcaption}
\usepackage{amsmath}
\usepackage{algorithmic}

 \newtheorem{thm}{Theorem}[section]
\newtheorem{definition2}[thm]{Definition}
\newtheorem{example2}[thm]{Example}
\newcommand{\ten}[1]{\mathcal{#1}}
\newcommand{\mat}[1]{\mathbf{#1}}
\newcommand{\bmat}[1]{\boldsymbol{#1}}
\newcommand{\ccf}[1]{{\color{black} #1}}

\newcommand{\zz}[1]{\textcolor{black}{#1}}

\begin{tcbverbatimwrite}{tmp_\jobname_header.tex}
\headers{\scriptsize Active Subspace of Neural Networks: Structural Analysis and Universal Attacks}{C. Cui, K. Zhang, T. Daulbaev, J. Gusak, I. Oseledets, and Z. Zhang}

\title{Active Subspace of Neural Networks: Structural Analysis and Universal Attacks \thanks{Submitted to the editors on October 2019. 
\funding{Chunfeng Cui, Kaiqi Zhang, and Zheng Zhang are supported by the UCSB start-up grant. 
Talgat Daulbaev, Julia Gusak, and Ivan Oseledets are supported by  the Ministry of Education and Science of the Russian Federation (grant 14.756.31.0001).
}}}

\author{Chunfeng Cui\thanks{University of California Santa Barbara, Santa Barbara, CA, USA 
  (\email{chunfengcui@ucsb.edu}, \email{kzhang07@ucsb.edu}, 
  \email{zhengzhang@ece.ucsb.edu}).}
  \and Kaiqi Zhang\footnotemark[2]
  \and Talgat Daulbaev\thanks{Skolkovo Institute of Science and Technology, Moscow, Russia
  (\email{talgat.daulbaev@skoltech.ru}, 
  \email{y.gusak@skoltech.ru}).}
  \and Julia Gusak\footnotemark[3]
  \and \newline
  Ivan Oseledets\thanks{Skolkovo Institute of Science and Technology and Institute of Numerical Mathematics of Russian Academy of Sciences,  Moscow, Russia (\email{ivan.oseledets@gamil.com}.)
  } 
  \and Zheng Zhang\footnotemark[2]
  }
\end{tcbverbatimwrite}
\headers{\scriptsize Active Subspace of Neural Networks: Structural Analysis and Universal Attacks}{C. Cui, K. Zhang, T. Daulbaev, J. Gusak, I. Oseledets, and Z. Zhang}

\title{Active Subspace of Neural Networks: Structural Analysis and Universal Attacks \thanks{Submitted to the editors on October 2019.
\funding{Chunfeng Cui, Kaiqi Zhang, and Zheng Zhang are supported by the UCSB start-up grant.
Talgat Daulbaev, Julia Gusak, and Ivan Oseledets are supported by  the Ministry of Education and Science of the Russian Federation (grant 14.756.31.0001).
}}}

\author{Chunfeng Cui\thanks{University of California Santa Barbara, Santa Barbara, CA, USA
  (\email{chunfengcui@ucsb.edu}, \email{kzhang07@ucsb.edu},
  \email{zhengzhang@ece.ucsb.edu}).}
  \and Kaiqi Zhang\footnotemark[2]
  \and Talgat Daulbaev\thanks{Skolkovo Institute of Science and Technology, Moscow, Russia
  (\email{talgat.daulbaev@skoltech.ru},
  \email{y.gusak@skoltech.ru}).}
  \and Julia Gusak\footnotemark[3]
  \and \newline
  Ivan Oseledets\thanks{Skolkovo Institute of Science and Technology and Institute of Numerical Mathematics of Russian Academy of Sciences,  Moscow, Russia (\email{ivan.oseledets@gamil.com}.)
  }
  \and Zheng Zhang\footnotemark[2]
  }

\begin{document}
\maketitle

\begin{tcbverbatimwrite}{tmp_\jobname_abstract.tex}
\begin{abstract}
Active subspace is a model reduction method widely used in the uncertainty quantification community. In this paper, we propose  analyzing the internal structure and vulnerability of deep neural networks using active subspace.  
Firstly, we employ the active subspace to measure the number of ``active neurons'' at each intermediate layer, \ccf{which indicates} that the number of  neurons can be reduced from several thousands to several dozens. This motivates us to change the network structure and to develop a new and more compact network, referred to as {ASNet}, that has significantly fewer model parameters.  Secondly, we propose  analyzing the vulnerability of a neural network using active subspace \ccf{by} finding an additive universal adversarial attack vector that can misclassify a dataset with a high probability.  
Our experiments on CIFAR-10 show that ASNet can  achieve 23.98$\times$ parameter and 7.30$\times$ flops reduction. The universal active subspace  attack vector can achieve around 20\% higher attack ratio compared with the existing approaches in our numerical experiments.  
The PyTorch codes for this paper are available online \footnote{Codes are available at: \url{https://github.com/chunfengc/ASNet}}.


\end{abstract} 

\begin{keywords}
 Active Subspace, Deep Neural Network, Network Reduction, Universal Adversarial Perturbation
\end{keywords}

\begin{AMS}
  90C26, 15A18, 62G35
\end{AMS}
\end{tcbverbatimwrite}
\begin{abstract}
Active subspace is a model reduction method widely used in the uncertainty quantification community. In this paper, we propose  analyzing the internal structure and vulnerability of deep neural networks using active subspace.
Firstly, we employ the active subspace to measure the number of ``active neurons'' at each intermediate layer, \ccf{which indicates} that the number of  neurons can be reduced from several thousands to several dozens. This motivates us to change the network structure and to develop a new and more compact network, referred to as {ASNet}, that has significantly fewer model parameters.  Secondly, we propose  analyzing the vulnerability of a neural network using active subspace \ccf{by} finding an additive universal adversarial attack vector that can misclassify a dataset with a high probability.
Our experiments on CIFAR-10 show that ASNet can  achieve 23.98$\times$ parameter and 7.30$\times$ flops reduction. The universal active subspace  attack vector can achieve around 20\% higher attack ratio compared with the existing approaches in our numerical experiments.
The PyTorch codes for this paper are available online \footnote{Codes are available at: \url{https://github.com/chunfengc/ASNet}}.


\end{abstract}

\begin{keywords}
 Active Subspace, Deep Neural Network, Network Reduction, Universal Adversarial Perturbation
\end{keywords}

\begin{AMS}
  90C26, 15A18, 62G35
\end{AMS}



\section{Introduction}


Deep neural networks have achieved impressive performance in many applications, such as computer vision \cite{krizhevsky2012imagenet}, nature language processing \cite{young2018recent}, and speech recognition \cite{graves2006connectionist}. Most neural networks   use  deep structure (i.e., many layers) and a huge number of neurons   to achieve a high accuracy and expressive  power~\cite{oymak2019towards,ge2019mildly}. 
 However, it is still unclear how many layers and neurons are necessary. Employing an unnecessarily complicated deep neural network can cause huge extra costs in run-time and hardware resources. Driven by resource-constrained applications such as robotics and internet of things, there is an increasing interest in building smaller neural networks by  removing network redundancy.  Representative methods include network pruning and  sharing~\cite{frankle2018lottery, han2015deep, he2018amc, liu2018rethinking, liu2018dynamic}, low-rank matrix and tensor factorization~\cite{sainath2013low, hawkins2019bayesian, garipov2016ultimate, lebedev2014speeding,novikov2015tensorizing}, parameter quantization~\cite{courbariaux2016binarized, deng2018gxnor}, knowledge distillation~\cite{hinton2015distilling, romero2014fitnets}, and so forth.  
However, most existing methods delete model parameters directly without changing the network architecture~\cite{he2018amc, han2015deep, cai2018proxylessnas, liu2018dynamic}.

Another important issue of deep neural networks is the lack of robustness. 
A deep neural network is desired to maintain good performance for noisy or corrupted data   to be deployed in safety-critical applications such as autonomous driving and medical image analysis. 
However, recent studies  have revealed that many state-of-the-art deep neural networks are vulnerable  to small perturbations  \cite{szegedy2013intriguing}.  
 A substantial number of methods have been proposed to generate adversarial examples. Representative works can be classified into four classes  \cite{serban2018adversarial}, including 
 optimization methods \cite{carlini2017towards,moosavi2016deepfool,moosavi2017universal, szegedy2013intriguing}, 
 sensitive features \cite{goodfellow2014explaining,papernot2016limitations}, geometric transformations \cite{dziugaite2016study,kanbak2018geometric}, 
 and generative models \cite{baluja2017adversarial}. 
However, these methods share a fundamental limitation: each perturbation is designed for a given data point, and one has to implement the algorithm again   to generate the perturbation for a new data sample. 
Recently, several methods have also been proposed  to compute a universal adversarial attack   to fool a dataset simultaneously (rather than one data sample) in various applications, such as computer vision \cite{moosavi2017universal}, speech recognition \cite{neekhara2019universal}, audio \cite{abdoli2019universal}, and text classifier \cite{behjati2019universal}. 
However,  all the above methods only solve a series of  data-dependent sub-problems. 
In \cite{khrulkov2018art}, Khrulkov et al. proposed to construct   universal perturbation by computing the so-called $(p,q)$-singular vectors of the Jacobian matrices of hidden layers of a network. 

This paper investigates the above two issues with the active subspace method~\cite{russi2010uncertainty,constantine2015active, constantine2014active} that was originally developed for uncertainty quantification. The key idea of the active subspace is to identify the \ccf{low-dimensional subspace constructed by some} important directions that can contribute significantly to the variance of the multi-variable function. 
These directions are \ccf{corresponding to} the principal components of the uncentered  covariance matrix of gradients. 
Afterwards, a response surface  can be constructed in this low-dimensional subspace to reduce the  number of parameters for  partial  differential equations  \cite{constantine2014active} and uncertainty quantification \cite{constantine2015exploiting}.  However, the power of active subspace in analyzing and attacking deep neural networks has not been explored.


\subsection{Paper Contributions}  
The contribution of this manuscript is twofold.

\begin{itemize}[leftmargin=*]
 \item Firstly, we apply the active subspace to  some intermediate layers of a deep neural network, and try to \ccf{answer} the following question: {\it how many neurons and layers are  important in a deep neural network?} 
Based on the active subspace,  we  propose the definition of  ``active neurons''. 
 Fig.~\ref{fig:motivation}~(a) shows that even though  there are tens of thousands of neurons, only dozens of them are important from  the active subspace point of view.  
Fig.~\ref{fig:motivation}~(b) further shows that most of the neural network parameters  are distributed in the last few layers. This motivates us to cut off the tail  layers  and replace  them with a smaller and simpler new framework called  ASNet. ASNet  contains three parts: the first few layers of a deep neural network, an active-subspace  layer that maps the intermediate  neurons to a  low-dimensional  subspace, and a  polynomial chaos expansion  layer that projects the reduced variables to the outputs. 
Our numerical experiments show that the proposed ASNet has  much  fewer model parameters  than the original one.  
ASNet can also be combined with existing structured re-training methods (e.g., pruning and quantization) to get   better accuracy \ccf{while using} fewer model parameters.

\item Secondly,  we \ccf{use} active subspace to develop a new universal attack method to fool deep neural networks on a whole data set. 
 We formulate this problem as a ball-constrained loss maximization problem and propose a heuristic projected  gradient descent algorithm to solve it. At each iteration, the ascent direction is the dominant  active subspace, and the stepsize is  decided by the backtracking algorithm.  
 Fig.~\ref{fig:motivation}~(c) shows that  \ccf{the attack ratio of the active subspace direction is much higher than that of  the random vector.}


\end{itemize}

\begin{figure}[t]
\centering
\includegraphics[width=
\textwidth]{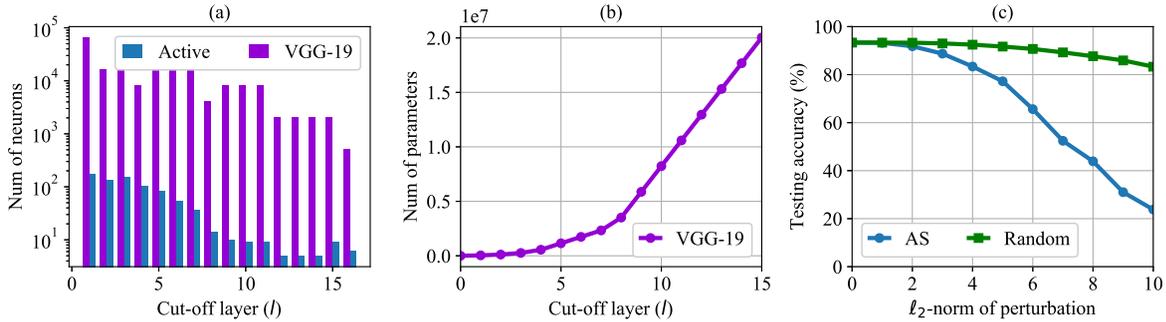}
\caption{Structural analysis of deep neural networks by the active subspace (AS). All  experiments are conducted on  CIFAR-10 by VGG-19. (a)  The number of neurons can be significantly reduced by the active subspace. 
\ccf{Here, the number of active neurons is defined by Definition~\ref{def:as} with a threshold $\epsilon=0.05$}; (b)  Most of the parameters are distributed in the last few layers; (c) The active subspace direction can perturb the network   significantly.} 
\label{fig:motivation}
\end{figure}

The rest of this manuscript is 
organized as follows. In Section~\ref{sec:AS}, we review the key idea of active subspace.  Based on the active-subspace method, Section~\ref{sec:ASNet}  shows how to find the number of active neurons in a deep neural network and further proposes a new and compact network, referred to as ASNet. Section~\ref{sec:adv} develops a new universal adversarial attack method based on active subspace. The numerical experiments for both ASNet and universal adversarial attacks are presented in Section~\ref{sec:Numerical}. Finally, we conclude this paper in Section~\ref{sec:conclusion}.

\section{Active Subspace}
\label{sec:AS} 
Active-subspace   is  an efficient tool for functional analysis and  dimension reduction. 
Its key idea is to construct a low-dimensional subspace for the input variables in which the function value changes dramatically.   
 Given a continuous function $\ccf{c}(\mat x$) with $\mat{x}$ described by the probability density function $\rho(\mat x)$, one can construct an uncentered covariance matrix for the gradient: $\mat C=\mathbb{E}[\nabla \ccf{c}(\mat x)\nabla \ccf{c}(\mat x)^T]$. 
Suppose the  matrix  $\mat C$ admits  the following eigenvalue decomposition,  
\begin{equation}
\label{eq:C_svd}
\mat C=\mat V\boldsymbol{\Lambda} \mat V^T,
\end{equation} 
where  $\mat V$ includes all orthogonal eigenvectors and  
\begin{equation}\label{equ:lmd} 
\boldsymbol{\Lambda}= {\rm diag} (\lambda_1, \cdots, \lambda_n), \ \lambda_1\geq  \cdots \geq \lambda_n \geq 0
\end{equation}
are the eigenvalues. 
All the eigenvalues are nonnegative because   $\mat C$ is \ccf{positive} semidefinite. One can split the matrix  $\mat V$ into two parts, 
 \begin{equation}\label{equ:V_split}
     \mat{V}=[\mat{V}_1,\ \mat{V}_2], \text{ where } \mat{V}_1\in\mathbb{R}^{n\times r} \text{ and } \mat{V}_2\in\mathbb{R}^{n\times (n-r)}.
 \end{equation} 
The subspace spanned by matrix  $\mat{V}_1\in\mathbb{R}^{n\times r}$ is called an active subspace~\cite{russi2010uncertainty}, because $\ccf{c}(\mat x)$ is  sensitive to perturbation vectors inside this subspace .

  \begin{remark}[Relationships with the Principal Component Analysis]
Given a set of data $\mat X=[\mat x^1,\ldots,\mat x^m]$ with each column representing a data sample and each row is zero-mean, the first principal component $\mat w_1$   inherits the maximal  variance from $\mat X$, namely, 
 \begin{equation}\label{equ:PCA_1}
     \mat w_1=\underset{\|\mat w\|_2=1}{\text{argmax}}\ \sum_{i=1}^m \mat (\mat w_1^T\mat x^i)^2 = \underset{\|\mat w\|_2=1}{\text{argmax}}\ \mat w^T\mat X\mat X^T\mat w.
 \end{equation}
 The variance is maximized when $\mat w_1$ is the eigenvector associated with the largest eigenvalue of $\mat X\mat X^T$. 
  The first $r$ principal components  are the $r$ eigenvectors associated with the $r$ largest eigenvalues of $\mat X\mat X^T$. 
  The main difference with the active subspace is that the principal component analysis uses the covariance matrix of input data sets $\mat X$, but the active-subspace method uses the covariance matrix of gradient $\nabla \ccf{c}(\mat x)$. 
  Hence, a perturbation along the direction $\mat{w}_1$ from \ccf{(\ref{equ:PCA_1}) only guarantee the variability in the data, and}  does not necessarily cause a significantly change on the value of $\ccf{c}(\mat{x})$. 
  
 
 \end{remark}

The following lemma quantitatively describes that $\ccf{c}(\mat x)$ varies more on average along the directions defined by the columns of $\mat V_1$ than   the directions defined by the columns of $\mat V_2$.


 \begin{lemma}\cite{constantine2014active}\label{lem:as_gradient}
 Suppose $\ccf{c}(\mat x)$ is a continuous function and $\mat C$ is obtained from (\ref{eq:C_svd}). For the matrices $\mat V_1$ and $\mat V_2$  generated by (\ref{equ:V_split}), and  the reduced vector 
 \begin{equation}
     \mat z=\mat V_1^T\mat x  \text{ and }  \tilde{\mat z}=\mat V_2^T\mat x,
 \end{equation} 
 it holds that 
\begin{align}
  \nonumber  \mathbb{E}_{\mat x}[\nabla_{\mat z} \ccf{c}(\mat x)^T\nabla_{\mat z} \ccf{c}(\mat x)] =& \lambda_1+\ldots+\lambda_r,\\ 
  \mathbb{E}_{\mat x}[\nabla_{\tilde{\mat z}} \ccf{c}(\mat x)^T\nabla_{\tilde{\mat z}} \ccf{c}(\mat x)] =& \lambda_{r+1}+\ldots+\lambda_n.
\end{align}
 \end{lemma}
\textsl{Sketch of proof \cite{constantine2014active}: }
\begin{align*}
     &\mathbb{E}_{\mat x}[\nabla_{\mat z} \ccf{c}(\mat x)^T\nabla_{\mat z} \ccf{c}(\mat x)]\\
    =&\text{trace}\left(\mathbb{E}_{\mat x}[\nabla_{\mat z} \ccf{c}(\mat x)\nabla_{\mat z} \ccf{c}(\mat x)^T] \right)\\
    =&\text{trace}\left( \mathbb{E}_{\mat x}[\mat V_1^T\nabla_{\mat x} \ccf{c}(\mat x)\nabla_{\mat x} \ccf{c}(\mat x)^T\mat V_1]\right)\\
    =&\text{trace}\left(\mat V_1^T\mat C\mat V_1\right)\\
    =&\lambda_1+\ldots+\lambda_r.
\end{align*}

When $\lambda_{r+1}=\ldots=\lambda_n=0$, Lemma \ref{lem:as_gradient} implies $\nabla_{\tilde{\mat z}} \ccf{c}(\mat x)$ is zero everywhere, i.e., $\ccf{c}(\mat x)$ is \ccf{$\tilde z$}-invariant. 
In this case, we may reduce   $\mat x\in\mathbb R^n$ to a low-dimensional vector $\mat z=\mat V_1^T\mat x\in\mathbb R^r$ and  construct  a new   response surface  $g(\mat z)$  to represent $\ccf{c}(\mat {x})$.   
Otherwise, if $\lambda_{r+1}$ is small, we may still construct a response surface $g(\mat z)$  to approximate $\ccf{c}(\mat x)$ with a bounded error, as shown in the following lemma.



\subsection{Response Surface} 
For a fixed $\mat z$, the best guess for \ccf{$g$} is the conditional expectation of $\ccf{c}$ given $\mat z$, i.e., 
\begin{equation}\label{equ:theory_g}
   g(\mat{z})=\mathbb{E}_{\tilde{\mat z}}[\ccf{c}(\mat {x})|\mat{z}]=\int \ccf{c}(\mat  V_1\mat z+\mat V_2\tilde{\mat z})\rho(\tilde{\mat z}|\mat z)d\tilde{\mat z}.
\end{equation} 
Based on the Poincar\'e inequality, the following approximation error bound is obtained~\cite{constantine2014active}. 

\begin{lemma}
\label{lem:astheory}
Assume that $\ccf{c}(\mat  x)$ is absolutely continuous and square integrable with respect to the probability density function $\rho(\mat x)$, then the  approximation function $g(\mat z)$ in (\ref{equ:theory_g}) satisfies: 
\begin{equation}\label{equ:approximationerr}
    \mathbb{E}[\left(\ccf{c}(\mat {x})-g(\mat{z})\right)^2]\le \ccf{O}(\lambda_{r+1}+\ldots+\lambda_n).  
\end{equation}
\end{lemma}

\textsl{Sketch of proof \cite{constantine2014active}: }
\begin{align*}
     &\mathbb{E}_{\mat x}[\left(\ccf{c}(\mat {x})-g(\mat{z})\right)^2]\\
     =& \mathbb{E}_{\mat z}[\mathbb{E}_{\tilde{\mat z}}[\left(\ccf{c}(\mat x)-g(\mat z)\right)^2 |\mat{z}]]\\
     \le&\text{const}\times \mathbb{E}_{\mat z}[\mathbb{E}_{\tilde{\mat z}}[\nabla_{\tilde{\mat z}} \ccf{c}(\mat x)^T\nabla_{\tilde{\mat z}} \ccf{c}(\mat x) |\mat{z}]] \text{\quad (Poincar\'e inequality)}\\
     =& \text{const}\times\mathbb{E}_{\mat x}[\nabla_{\tilde{\mat z}} \ccf{c}(\mat x)^T\nabla_{\tilde{\mat z}} \ccf{c}(\mat x)]\\
     = &\text{const}\times (\lambda_{r+1}+\ldots+\lambda_n)\text{\quad (Lemma \ref{lem:as_gradient})}\\
     =& \ccf{O(\lambda_{r+1}+\ldots+\lambda_n).}
\end{align*}

In other words, the active-subspace approximation error will be small if $\lambda_{r+1},\ldots,\lambda_n$ are negligible. 
 

\section{Active Subspace for Structural Analysis and Compression of  Deep Neural Networks}
\label{sec:ASNet} 

This section  applies the  active subspace to analyze the internal layers of a deep neural network to  reveal the number of important neurons at each layer. Afterward, a new network called ASNet is built to reduce the storage and computational complexity. 

\subsection{Deep Neural Networks}

A deep neural network can be described as 
\begin{equation}
    f(\mat x_0)=f_{L}\left (f_{L-1}\ldots \left( f_1(\mat x_0)\right) \right),
\end{equation}
where $\mat x_0\in\mathbb{R}^{n_0}$ is an input, $L$ is the total number of layers, and  $f_{l}: \mathbb{R}^{n_{l-1}}\rightarrow  \mathbb{R}^{n_{l}}$ is a function representing the $l$-th layer (e.g., combinations of convolution \ccf{or} fully connected, batch normalization, ReLU, or pooling layers).  
For any  $1\le l\le L$, we rewrite the above feed-forward model  as a superposition of functions, i.e., 
\begin{equation}
   f(\mat x_0) = f^{\ccf l}_{\text{post}}(f^{\ccf l}_{\text{pre}}(\mat x_0)),
\end{equation} 
where the {\bf pre-model} $f^{\ccf l}_{\text{pre}}(\cdot)=f_{l}\ldots (f_{1}(\cdot))$  denotes all operations before the $l$-th layer and the {\bf post-model} $f^{\ccf l}_{\text{post}}(\cdot)=f_{L}\ldots (f_{l+1}(\cdot))$  denotes  all succeeding operations.  
The intermediate neuron $\mat x_{\ccf l} = f^{\ccf l}_{\text{pre}}(\mat x_0)  \in\mathbb{R}^{n_{l}}$   usually lies in a high dimension. We aim to study whether such a high dimensionality is necessary. If not, how can we reduce it?

\subsection{The Number of Active Neurons}

Denote $\text{loss}(\cdot)$ as the loss function, and 
\begin{equation}
    c_{\ccf l}(\mat  x)=\text{loss}(f^{\ccf l}_{\text{post}}(\mat x)).
\end{equation}  
The covariance matrix $\mat C=\mathbb{E}[\nabla c_{\ccf l}(\mat  x) \nabla c_{\ccf l}(\mat  x)^T]$   admits    the eigenvalue decomposition $\mat C=\mat V\boldsymbol{\Lambda} \mat V^T$ with $\boldsymbol{\Lambda}= {\rm diag} (\lambda_1, \cdots, \lambda_{n_l})$.   
We try to extract the active subspace of $c_{\ccf l}(\mat x)$ and reduce the intermediate vector $\mat x$ to a \ccf{low} dimension. 
\ccf{Here the intermediate neuron $\mat x$, the covariance matrix $\mat C$, eigenvalues $\boldsymbol \Lambda$, and eigenvectors $\mat V$ are also related to the layer index $l$,  but we ignore the index for simplicity.}

\begin{definition2} 
\label{def:as}
Suppose $\boldsymbol \Lambda$ is computed by  (\ref{equ:lmd}).  For any layer index $1\le l\le L$, we define \textbf{the number of active neurons} $n_{l,\text{AS}}$ as follows: 
\begin{equation}
\label{equ:nAS}
    n_{l,AS} = \arg\min\left\{i:\   \frac{\lambda_1+\ldots+\lambda_i}{\lambda_1+\ldots+\lambda_{n_l}}\ge 1-\epsilon\right\},
\end{equation}
where $\epsilon>0$ is a user-defined threshold. 
\end{definition2}

Based on Definition \ref{def:as}, the  post-model can be approximated by an  $n_{l,\text{AS}}$-dimensional  function  with a high accuracy, i.e.,  
\begin{equation}\label{equ:app}
      g_{\ccf l}(\mat{z})=\mathbb{E}_{\tilde{\mat z}}[c_{\ccf l}(\mat {x})|\mat{z}].
\end{equation}
Here  $\mat z=\mat V_1^T\mat x \in\mathbb{R}^{n_{l,AS}}$ plays the role of active neurons, $\tilde{\mat z}=\mat V_2^T\mat x \in\mathbb{R}^{n-n_{l,AS}}$,  and $\mat V=[\mat V_1, \mat V_2]$.

\begin{lemma}
\label{lem: ASerror}
Suppose  the input $\mat x_0$ is  bounded.  Consider a deep neural network  with the following operations: convolution, \ccf{fully connected}, 
  ReLU, batch normalization,   max-pooling, and equipped with the cross entropy   loss function. Then for any   $l\in\{1,\ldots,L\}$, $\mat x=f^{\ccf l}_{\text{pre}}(\mat x_0)$, and $c_{\ccf l}(\mat x)=\text{loss}(f^{\ccf l}_{\text{post}}(\mat x))$,  the \ccf{$n_{l,AS}$-dimensional} function $g_{\ccf l}(\mat z)$ defined in (\ref{equ:app}) satisfies 
\begin{equation}\label{equ:approximationerr2}
 \mathbb{E}_{\ccf{\mat z}}\left[\left(g_{\ccf l}(\mat{z})\right)^2\right] \le 2\mathbb{E}_{\ccf{\mat x_0}}\left[\left(c_{\ccf 0}(\mat {x_0})\right)^2\right] + O(\epsilon). 
\end{equation}
\end{lemma}

\begin{proof}
   Denote $c_{\ccf l}(\mat x) = \text{loss}(f_L(\ldots (f_{l+1}(\mat x)))$,  where $\text{loss}(\mat y)=-\log \frac{\exp(y_{b})}{\sum_{i=1}^{\ccf{n_{L}}} \exp(y_i)}$ is the  cross entropy loss function, $b$ is the true label, and $\ccf{n_L}$ is the total number of  classes.  
We first  show $c_{\ccf l}(\mat x)$ is  absolutely continuous and square integrable, and then apply  Lemma~\ref{lem:astheory} to derive (\ref{equ:approximationerr2}).

Firstly,   all components of $c_{\ccf l}(\mat x)$ are Lipschitz continuous   because  (1) the convolution, \ccf{fully connected}, and batch normalization operations are all linear; (2) the max pooling and ReLU functions are non-expansive. Here, a mapping $m$ is non-expansive if $\|m(\mat x)-m(\mat y)\|\le\|\mat x-\mat y\|$; (3) the cross entropy loss function is smooth with an  upper bounded gradient, i.e.,  $\|\nabla \text{loss}(\mat  y)\|=\|\mat e_b-\exp(\mat y)/\sum_{i=1}^{\ccf n_L} \exp(y_i)\|\le \sqrt{n_{L}}$.  
The composition of two Lipschitz \zz{continuous} functions is also be Lipschitz continuous: suppose the Lipschitz constants for $f_1$ and $f_2$ are  $\alpha_1$ and $\alpha_2$, respectively, it holds that $\|f_1(f_2(\bar{\mat x}))-f_1(f_2(\underline{\mat x}))\|\le \alpha_1\|f_2(\bar{\mat x})-f_2(\underline{\mat x})\|\le \alpha_1\alpha_2\|\bar{\mat x}-\underline{\mat x}\|$ for any vectors $\bar{\mat x}$ and $ \underline{\mat x}$.  
By recursively applying the above rule,  $c_{\ccf l}(\mat x)$ is Lipschitz  continuous:  
\begin{align*}
 \|c_{\ccf l}(\bar{\mat x})-c_{\ccf l}( \underline{\mat x})\|_2&=\|\text{loss}(f_L(\ldots (f_{l+1}(\bar{\mat x})))) - \text{loss}(f_L(\ldots (f_{l+1}(\underline{\mat x})))) \|_2\\
 &\le  \ccf{\sqrt{n_L}}\alpha_{L}\ldots \alpha_{l+1}\|\bar{\mat x}-\underline{\mat x}\|_2. 
\end{align*} 
The intermediate neuron $\mat x$ is in a bounded domain because the input $\mat x_0$ is bounded and all functions $f_i(\cdot)$ are either continuous or non-expansive.  
Based on the fact that any  Lipschitz-continuous function is also  absolutely continuous on a compact domain \cite{realanalysis2010},   we conclude that $c_{\ccf l}(\mat x)$ is   absolutely continuous. 
 
 Secondly,  because $\mat x$  is bounded and $c_{\ccf l}(\mat x)$ is continuous,  both $c_{\ccf l}(\mat x)$ and its  square integral will be bounded, i.e., 
$\int (c_{\ccf l}(\mat x)^2\rho(\mat x)d\mat x <\infty$. 

Finally, by Lemma~\ref{lem:astheory}, it holds that 
\begin{equation*} 
    \mathbb{E}_{\ccf{\mat x}}[(c_{\ccf l}(\mat {x})-g_{\ccf l}(\mat{z}))^2]\le \ccf{O}(\lambda_{n_{l,AS}+1}+\ldots+\lambda_n). 
\end{equation*}
\ccf{From Definition~\ref{def:as}, we have
\begin{equation*}
\lambda_{n_{l,AS}+1}+\ldots+\lambda_n \le (\lambda_1+\ldots+\lambda_n)\epsilon = \|\mat C^{1/2}\|_F^2\epsilon=O(\epsilon).
\end{equation*}
In the last equality, we used that $\|\mat C^{1/2}\|_F$ is upper bounded because  $c_l(\mat x)$ is Lipschitz continuous with a bounded gradient.}
\ccf{
Consequently, we have 
\begin{align*}
 & \mathbb{E}_{\ccf{\mat x}}[(g_{\ccf l}(\mat{z}))^2]  \\
 =& \mathbb{E}_{\ccf{\mat x}}[(g_{\ccf l}(\mat{z})-c_{\ccf l}(\mat {x}) + c_{\ccf l}(\mat {x}))^2] \\
 \le &  2\mathbb{E}_{\ccf{\mat x}}[(c_{\ccf l}(\mat {x}))^2] + 2 \mathbb{E}_{\ccf{\mat x}}[(c_{\ccf l}(\mat {x})-g_{\ccf l}(\mat{z}))^2]\\
 = & 2\mathbb{E}_{\ccf{\mat x_0}}[(c_{\ccf 0}(\mat {x_0}))^2] + 2 \mathbb{E}_{\ccf{\mat x}}[(c_{\ccf l}(\mat {x})-g_{\ccf l}(\mat{z}))^2]\\
 \le &2\mathbb{E}_{\ccf{\mat x_0}}[(c_{\ccf 0}(\mat {x_0}))^2]  +  O(\epsilon).
\end{align*}
}
The proof is completed.  
\end{proof}

The above lemma shows that  the active subspace method can   reduce the number of neurons of the $l$-th layer from $n_{l}$ to $n_{l, AS}$. 
 \ccf{The loss for the low-dimensional function $g_l(\mat z)$ is bounded by two terms: the loss  $c_0(\mat x_0)$ of the original network, and the threshold  $\epsilon$ related to $n_{l,AS}$. 
 This loss function is the cross entropy loss, not the classification error. 
 However, it is believed that a small loss will result in a small classification error. 
 Further, the result in Lemma \ref{lem: ASerror} is valid for thr fixed parameters in the pre-model. 
 In practice, we can fine-tune the pre-model to achieve   better accuracy. 
 }

\ccf{Further, a small  active neurons $n_{l,AS}$ is critical to get a high compress ratio. 
From  Definition~\ref{def:as},  $n_{l,AS}$ depends on the eigenvalue distribution of the covariance matrix $\mat C$. 
For a proper network structure and a good choice of the layer index $l$, if the eigenvalues of $\mat C$ are dominated by the first few eigenvalues, 
then $n_{l,AS}$ will be small. For instance, in  Fig.~\ref{fig:lmd_PCA_LR}(a), the   eigenvalues for layers $4\le l\le7$ of VGG-19 are nearly exponential decreasing to zero. 
}

 
\subsection{Active Subspace Network (ASNet)}
 
 \begin{algorithm}[t]
\caption{The training procedure of the active subspace network ({ASNet})}
\textbf{Input: } {A pretrained deep neural network, the layer index  $l$, and the number of active neurons $r$.}

\begin{enumerate}
    \item[Step 1] \textbf{Initialize the active subspace layer.} 
    The active subspace layer is a linear projection where the  projection matrix $\mat V_1\in\mathbb{R}^{n\times r}$  is computed  by Algorithm~\ref{alg:onlineASMat}.
    If  $r$ is not given, we use $r=n_{\text{AS}}$ defined in (\ref{equ:nAS}) by default. 
    \item[Step 2] \textbf{Initialize the polynomial chaos expansion layer.} The polynomial chaos expansion layer is a nonlinear mapping from the reduced active subspace to the outputs, as shown in (\ref{equ:PCE}).  
    The weights $\mat c_{\boldsymbol \alpha}$ is computed by  (\ref{equ:coefficient}). 
    \item[Step 3] \textbf{Construct the ASNet.} Combine the  pre-model (the first $l$ layers of the deep neural network) with the active subspace and polynomial chaos expansion layers as a new network, referred to as  ASNet. 
    \item[Step 4] \textbf{Fine-tuning.} Retrain \ccf{all the parameters in pre-model, active subspace layer and polynomial chaos expansion layer in} ASNet for several epochs by stochastic gradient descent. 
\end{enumerate}

\textbf{Output:} {A new network {ASNet}}
\label{alg:ASNet}
\end{algorithm}

This subsection  proposes a new network called ASNet  that can   reduce both the storage and computational cost.  
Given a deep neural network, we first choose a proper layer   $l$ and project the high-dimensional intermediate neurons to a low-dimensional vector in the active subspace. Afterward,  the post-model is deleted completely and replaced with a nonlinear model that maps the low-dimensional active feature vector to the output directly. This new network,  called  ASNet, has  three parts:
\begin{itemize}
    \item[(1)] {\bf Pre-model:} the pre-model includes the first $l$  layers of a   deep neural network. 
    
    \item[(2)] {\bf Active subspace layer:}   a linear projection from the intermediate neurons  to the   low-dimensional active subspace.   
    
    \item[(3)] {\bf Polynomial chaos expansion layer:} the polynomial chaos expansion  \cite{ghanem1991stochastic, xiu2002wiener}   maps  the active-subspace variables to the output.   
\end{itemize}
\ccf{The initialization for the active subspace layer  and polynomial chaos expansion layer are presented in Sections \ref{sec:ASLayer} and~\ref{subsec:PCE}, respectively.  
We can also retrain all the parameters to increase the accuracy.}   
The whole procedure is illustrated in Fig.~\ref{fig:ASNet}~(b) and  Algorithm~\ref{alg:ASNet}.


 \def\x{3}

\begin{figure}[t]

\begin{subfigure}[b]{0.54\textwidth}
\centering
\begin{tikzpicture}[scale=0.44,
    roundnode/.style={circle, draw=black, fill=none, thick, minimum size=1mm},
    squarednode/.style={rectangle, draw=black, fill=none, very thick, minimum size=3mm},
    dummynode/.style={circle, draw=none, fill=none, very thick, minimum size=0.5mm},
    plainnode/.style={draw=none,fill=none},
    plainarrow/.style={->, thin},
    ]
    plainnoarrow/.style={-, thick},
    ]
    
    \foreach \i in {1,...,4} {
        \node[roundnode, fill=red!80!white] at (0, \i) (node1\i) {};
    }
    \foreach \i in {1,...,6} {
        \node[roundnode, fill=blue!50!white] at (\x, \i-1) (node2\i) {};
    }
    \foreach \i in {1,...,4} {
        \node[roundnode, fill=green!20!white] at (2*\x, \i) (node3\i) {};
    }
    \foreach \i in {1,...,3} {
        \node[roundnode, fill=yellow!50!white] at (3*\x, \i+0.5) (node4\i) {};
    }
    \foreach \i in {1,...,4} {
        \node[dummynode] at (4*\x, \i+0.5) (node5\i) {};
    }
    
    
    \foreach \i in {1,...,4 } {
        \foreach \j in {1,...,6} {
            \draw [plainarrow] (node1\i.east) -- (node2\j.west);
        }
    }
    \foreach \i in {1,...,6 } {
        \foreach \j in {1,...,4} {
            \draw [plainarrow] (node2\i.east) -- (node3\j.west);
        }
    }
    \foreach \i in {1,...,4 } {
        \foreach \j in {1,...,3} {
            \draw [plainarrow] (node3\i.east) -- (node4\j.west);
        }
    }
    \foreach \i in {1,...,3} {
        \draw [plainarrow] (node4\i.east) -- (node5\i.west);
    }
      \node[draw] at (1,6) {\tiny layer 1};  
  \node[draw] at (4.5,6) {\tiny layer 2};
  \node[draw] at (7.5,6) {\tiny ...};
  \node[draw] at (10,6) {\tiny layer L};
  
\end{tikzpicture} 
\caption{A deep neural network} 
\end{subfigure}
\begin{subfigure}[b]{0.44\textwidth}
\centering
\begin{tikzpicture}[scale=0.44,
    roundnode/.style={circle, draw=black, fill=none,  thick, minimum size=1mm},
    squarednode/.style={rectangle, draw=black, fill=none, very thick, minimum size=3mm},
    dummynode/.style={circle, draw=none, fill=none,  thick, minimum size=0.5mm},
    plainnode/.style={draw=none,fill=none},
    plainarrow/.style={->, thin},
    ]
    
    \foreach \i in {1,...,4} {
        \node[roundnode, fill=red!80!white] at (0, \i) (node1\i) {};
    }
    \foreach \i in {1,...,6} {
        \node[roundnode, fill=blue!50!white] at (\x, \i-1) (node2\i) {};
    }
    
   \foreach \i in {1,...,3} {
        \node[roundnode, fill=green!90!white] at (2*\x, \i+0.5) (node3\i) {};
    }
    
    \foreach \i in {1,...,3} {
        \node[dummynode] at (3*\x, \i+0.5) (node4\i) {};
    }
    
    \foreach \i in {1,...,3} {
        \node[dummynode, minimum size=0] at (2.5*\x, \i+0.5) (node35\i) {};
    }

    
    \foreach \i in {1,...,4 } {
        \foreach \j in {1,...,6} {
            \draw [plainarrow] (node1\i.east) -- (node2\j.west);
        }
    } 
    
    \foreach \i in {1,...,6 } {
    \foreach \j in {1,...,3} {
        \draw [plainarrow] (node2\i.east) -- (node3\j.west);
        }
    } 
    
    \foreach \i in {1,...,3} {
        \draw [plainarrow, -] (node3\i.east) to [bend right] (node35\i.east);
    }
    
    \foreach \i in {1,...,3} {
        \draw [plainarrow] (node35\i.east) to [bend left] (node4\i.west);
    }
  \node[draw] at (1,6) {\tiny pre-model};  
  \node[draw] at (4.5,6) {\tiny AS};
  \node[draw] at (7,6) {\tiny PCE};
\end{tikzpicture}

\caption{The proposed ASNet} 
\end{subfigure}

\caption{(a)  The original deep neural network;  (b)  The proposed ASNet with three parts: a pre-model, an active subspace (AS) layer, and a polynomial chaos expansion (PCE) layer.} 
\label{fig:ASNet}
\end{figure}
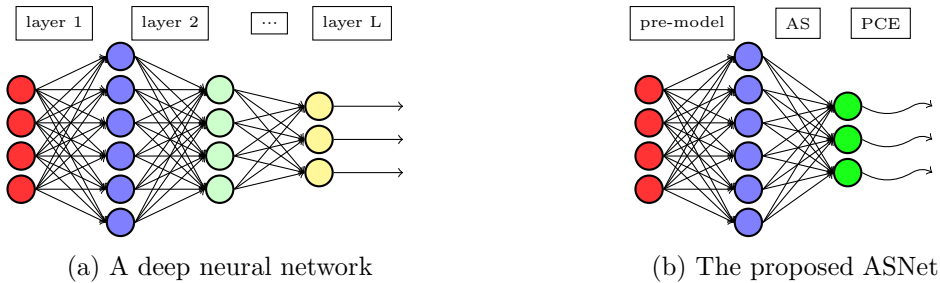

\subsection{The  Active Subspace Layer}
\label{sec:ASLayer}
 \ccf{This subsection presents an efficient method to project the high dimensional neurons to the active subspace.} 
Given a  dataset $\ten D=\{\mat x^1,\ldots,\mat x^m\}$, the empirical covariance matrix is computed by  $\hat{\mat{C}}=\frac1m\sum_{i=1}^m \nabla c_{\ccf l}(\mat  x^i)\nabla c_{\ccf l}(\mat  x^i)^T$. 
\ccf{When ReLU is applied as an activation, $c_l(\mat x)$ is not differentiable. 
In this case,  $\nabla$ denotes the sub-gradient with a little abuse of notation.}

Instead of calculating the eigenvalue decomposition of $\hat{\mat C}$,  we compute  the singular value decomposition   of $\hat{\mat{G}}$ to save the computation cost:
\begin{equation}
\label{equ:G}
   \hat{\mat{G}}  = [\nabla c_{\ccf l}(\mat {x}^1), \ldots, \nabla c_{\ccf l}(\mat {x}^m)]=\hat{\mat{V}}\hat{\bmat{\Sigma}}\hat{\mat{U}}^T   \in\mathbb{R}^{n_l\times m}  \ {\rm  with}\;  \hat{\bmat{\Sigma}}  ={\rm diag} (\hat{\sigma}_1, \cdots, \hat{\sigma}_{n_l}).
\end{equation}
The eigenvectors of $\mat C$ are approximated by the left singular vectors $\hat{\mat{V}}$ and the eigenvalues of $\mat C$ are approximated by the singular values of $\hat{\mat G}$, i.e.,  $\bmat\Lambda\approx\hat{\bmat\Sigma}^2$. 



We use the memory-saving frequent direction method \cite{ghashami2016frequent}  to compute the   $r$ dominant  singular value components, i.e., $\hat{\mat G}\approx\hat{\mat{V}}_r\hat{\bmat{\Sigma}}_r\hat{\mat{U}}_r^T$. Here $r$ is  smaller than the total number of samples. 
The frequent direction approach only stores an $n\times r$ matrix   $\mat S$.  
At the beginning, each column of  $\mat S\in\mathbb{R}^{n\times r}$ is initialized by   a gradient vector. 
Then    the randomized singular value  decomposition \cite{halko2011finding} is used   to generate $\mat S=\mat U\bmat \Sigma\mat V^T$.  
Afterwards, $\mat S$ is updated in the following way,
\begin{equation}\label{equ:svd_thres}
     \mat S\leftarrow\mat V\sqrt{\boldsymbol \Sigma^2-\sigma_r^2}.
\end{equation} 
 Now the last column of $\mat S$  is zero and we replace it with the gradient vector of a new sample. 
By repeating this process, $\mat S\mat S^T$ will approximate $\hat{\mat G} \hat{\mat G}^T$ with a high accuracy and $\mat V$ will approximate the left singular vectors of $\hat{\mat G}$.    
The algorithm framework is presented in  Algorithm~\ref{alg:onlineASMat}.


 \begin{algorithm}[t]
\caption{The frequent direction algorithm for computing the active subspace}
\label{alg:onlineASMat}

\textbf{Input: } A   dataset with $\ccf{m}_{AS}$ input samples $\{\mat x_0^j\}_{j=1}^{m_{AS}}$, a pre-model $f^{\ccf l}_{\text{pre}}(\cdot)$,  a  subroutine for computing   $\nabla c_l(\mat x)$,  and the dimension of truncated singular value decomposition  $r$. 

\begin{algorithmic}[1]
\STATE Select $r$   samples $\mat x_0^i$, compute $\mat x^i=f^{\ccf l}_{\text{pre}}(\mat x_0^i)$, and construct an initial  matrix $\mat S\leftarrow[\nabla c_l(\mat  x^1),\ldots,\nabla c_l(\mat x^r)]$. 
\FOR {t=1, 2, $\ldots,$}
\STATE Compute the singular value  decomposition $\mat V\boldsymbol\Sigma\mat U^T\leftarrow{\rm svd}(\mat S)$, where $\boldsymbol\Sigma=\text{diag}(\sigma_1,\ldots,\sigma_r)$. 
\STATE If the maximal number of samples $\ccf{m}_{AS}$ is reached, stop.  
\STATE Update $\mat S$ by the soft-thresholding (\ref{equ:svd_thres}). 
\STATE Get a new sample $\mat x_{0}^{\text{new}}$, compute $\mat x^{\text{new}} = f_{\text{pre}}^{\ccf l}(\mat x_{0}^{\text{new}})$, and replace the last column of $\mat S$ (now   all zeros) by the  gradient vector $\mat{S}(:,r)\leftarrow\nabla c_l(\mat  x^{\text{new}})$.
\ENDFOR

  \textbf{Output:} The  projection matrix $\mat V\in\mathbb{R}^{n_l\times r}$ and the singular values $\boldsymbol \Sigma\in\mathbb{R}^{r\times r}$. 
\end{algorithmic}
\end{algorithm}

After obtaining $\boldsymbol \Sigma=\text{diag}(\sigma_1,\ldots,\sigma_r)$, we can approximate the number of active neurons as 
\begin{equation}\label{equ:nAS_hat}
    \hat{n}_{l,AS}=\arg\min\left\{i:\quad\frac{\sqrt{\sigma_1^2+\ldots+\sigma_i^2}}{\sqrt{\sigma_1^2+\ldots+\sigma_r^2}}\ge1-\epsilon\right\}. 
\end{equation}
Under the condition that $\sigma_i^2\rightarrow \lambda_i$ for $i=1,\ldots,r$ and  $\lambda_{\ccf i}\rightarrow0$ for $i=r+1,\ldots,n_l$, 
(\ref{equ:nAS_hat}) can approximate $n_{l,AS}$ in  (\ref{equ:nAS}) with a high accuracy. 
Further, the projection matrix $\hat{\mat V}_1$ is chosen as the first $\hat n_{l,AS}$ columns of $\mat V$.
\ccf{The storage cost is reduced from $O(n_l^2)$ to $O(n_lr)$ and the computational cost is reduced from $O(n_l^2r)$ to $O(n_lr^2)$.}


\subsection{Polynomial Chaos Expansion Layer}
\label{subsec:PCE}

We continue to construct a new surrogate model to  approximate the post-model of a deep neural network.   
This  problem  can be regarded as an  uncertainty quantification problem if   we set $\mat z$ as  a random vector. 
\ccf{We choose the nonlinear polynomial because it has   higher expressive power than linear functions.}

By the polynomial chaos expansion    \cite{xiu2002modeling}, the network output $\mat y\in\mathbb{R}^{\ccf{n_L}}$ is approximated by  a linear combination of the  orthogonal polynomial basis functions: 
\begin{equation}\label{equ:PCE}
    \hat{\mat y}\approx  \sum_{|\bmat{\alpha}|=0}^p \mat{c}_{\bmat{\alpha}}\bmat{\phi}_{\bmat{\alpha}}(\mat z), \text{ where }|\bmat{\alpha}|=\alpha_1+\ldots+\alpha_d.
\end{equation} 
Here 
 $\bmat{\phi}_{\bmat{\alpha}}(\mat z)$ is a multivariate polynomial basis function chosen based on the probability density function of $\mat{z}$.  
When the parameters $\mat z=[z_1,\ldots,z_r]^T$ are independent,   both the joint density function  and the  multi-variable basis function can be decomposed into products of one-dimensional functions, i.e., $\rho(\mat z)=\rho_1(z_1)\ldots\rho_r(z_r)$,  $\bmat{\phi}_{\bmat{\alpha}}(\mat z) = \phi_{\alpha_1}(z_1) \phi_{\alpha_2}(z_2)\ldots \phi_{\alpha_r}(z_r).$  
The marginal basis function $\phi_{\alpha_j}(z_j)$ is uniquely determined by the marginal density function $\rho_i(z_i)$. The scatter plot in Fig.~\ref{fig:ASdist}  shows  that the marginal probability density of e$z_i$  is close to a Gaussian distribution. 

Suppose  $\rho_i(z_i)$ follows a Gaussian distribution, then  $\phi_{\alpha_j}(z_j)$ will be  a Hermite polynomial \cite{lide2018handbook}, i.e., 
\begin{equation}\label{equ:basis_Hermite}
    \phi_0(z)=1,\ \phi_1(z)=z,\ \phi_2(z)=4z^2-2,\ \phi_{p+1}(z)=2z\phi_{p}(z)-2p\phi_{p-1}(z). 
\end{equation} 
In   general, the elements in $\mat z$ can be  non-Gaussian correlated. In this case, the basis functions $\{ \bmat{\phi}_{\bmat{\alpha}}(\mat z) \}$ can be built via the Gram-Schmidt approach described in~\cite{cui2018stochastic}.

 \begin{figure}[t]
    \centering
\includegraphics[width=2.1in]{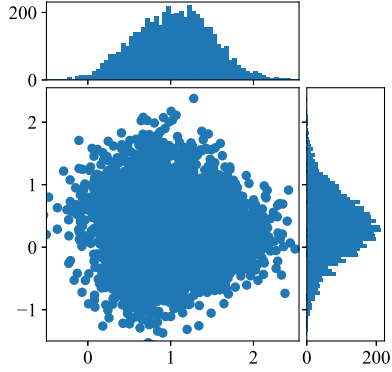}
    \caption{\small Distribution of the first two active subspace variables at the 6-th layer of VGG-19 for CIFAR-10. }
    \label{fig:ASdist}
\end{figure}

  

The coefficient $\mat{c}_{\bmat{\alpha}}$ can be computed by a linear least-square optimization.  
Denote $\mat z^j=\hat{\mat V}_1^Tf^{\ccf l}_{\text{pre}}(\mat x_0^j)$ as the random samples and $\mat y^j$ as the network output for  $j=1,\ldots,\ccf{m}_{\text{PCE}}$. The coefficient vector $\mat c_{\bmat{\alpha}}$ can be computed by 
\begin{equation}\label{equ:coefficient}
    \min_{\{\mat c_{\bmat{\alpha}}\}}\quad  \ccf{\frac{1}{m_{\text{PCE}}}}\sum_{j=1}^{m_{\text{PCE}}}
       \|\mat y^j- \sum_{|\bmat{\alpha}|=0}^p \mat{c}_{\bmat{\alpha}}\bmat{\phi}_{\bmat{\alpha}}(\mat z^j)\|^2.
\end{equation}
Based on the Nyquist-Shannon sampling theorem, the number of samples to train $\mat c_{\bmat \alpha}$ needs to satisfy   $m_{\text{PCE}}\ge 2n_{\text{basis}}=2\binom{r+p}{p}$. 
However, this number can be reduced to a smaller set of  ``important'' samples  by the D-optimal design  \cite{zankin2018gradient} or the sparse regularization approach \cite{cui2019high}.

The polynomial chaos expansion builds a  surrogate model to approximate the deep neural network output $\mat y$. 
This idea is similar to  the  knowledge  distillation \cite{hinton2015distilling}, where a pre-trained teacher network  teaches a  smaller student network to learn the  output feature.  
However, our polynomial-chaos layer uses one  nonlinear projection whereas the knowledge distillation uses a series of layers.  Therefore, the polynomial chaos expansion is more efficient in terms of computational and storage cost. 
\ccf{The polynomial chaos expansion layer is  different from the polynomial activation because the dimension of $\mat z$ may be different from that of output $\mat y$.}

\ccf{The problem (\ref{equ:coefficient}) is  convex   and any first order method can get a global optimal solution. 
Denote the optimal coefficients as $\mat c_{\alpha}^*$ and the finial objective value as $\delta^*$, i.e., 
\begin{equation}\label{equ:PCE_optim}
 \delta^* =  \ccf{\frac{1}{m_{\text{PCE}}}}\sum_{j=1}^{m_{\text{PCE}}}\|  \mat y^j-  \psi^*(\mat z^j)\|^2, \text{  \  where \ } \psi^*(\mat z^j)=\sum_{|\bmat{\alpha}|=0}^p \mat{c}^*_{\bmat{\alpha}}\bmat{\phi}_{\bmat{\alpha}}(\mat z^j).
\end{equation} 
If $\delta^*=0$, the polynomial chaos expansion is a good approximation to the original deep neural network on the training dataset. 
However, the approximation loss of the testing dataset may   be large because of the overfitting phenomena.}

\ccf{The objective function in (\ref{equ:coefficient}) is an empirical approximation to the expected error 
\begin{equation}\label{equ:exp_PCE_error}
\mathbb{E}_{(\mat z,\mat y)}[ \|\mat y- \psi(\mat z)\|^2 ], \text{  where  } \psi(\mat z)=\sum_{|\bmat{\alpha}|=0}^p \mat{c}_{\bmat{\alpha}}\bmat{\phi}_{\bmat{\alpha}}(\mat z).
\end{equation}
According to the Hoeffding's inequality \cite{hoeffding1994probability}, the expected error  (\ref{equ:exp_PCE_error}) is close to the empirical error (\ref{equ:coefficient}) with a high probability. 
Consequently,  the loss for ASNet with polynomial chaos expansion layer is bounded as follows. 

\begin{lemma}\label{lem:err_PCE}
Suppose that the optimal solution for solving problem (\ref{equ:coefficient}) is  $\mat c_{\alpha}^*$, the optimal polynomial chaos expansion is $\psi^*(\mat z)$, and the optimal residue is $\delta^*$. 
Assume that there exist consts $a, b$ such that for all $j$, $\|\mat y^j - \psi^*(\mat z^j)\|^2\in [a,b]$.  
Then the loss of ASNet will be upper bounded
\begin{equation}
\mathbb{E}_{\mat z}[(\text{loss}(\psi^*(\mat z)))^2]\le 2\mathbb{E}_{\mat x_0}[(c_0(\mat x_0))^2] + 2n_L(\delta^*+t)   \ \text{ w.p. } 1-\gamma^*,
\end{equation}
where $t$ is a user-defined threshold, and $\gamma^*=\exp(-\frac{2t^2m_{\text{PCE}}}{(b-a)^2})$. 
\end{lemma}

\begin{proof}
Since the cross entropy loss function is $\sqrt{n_L}$-Lipschitz continuous, we have  
\begin{equation}\label{equ:bound_Lip}
   \mathbb{E}_{(\mat y,\mat z)}[(\text{loss}(\mat y)- \text{loss}(\psi^*(\mat z)))^2] \le n_L  \mathbb{E}_{(\mat y,\mat z)}[\|\mat y - \psi^*(\mat z)\|^2], \end{equation}
Denote $\ten T^j = \|\mat y^j - \psi^*(\mat z^j)\|^2$ for $i=1,\ldots,n_L$. 
$\{\ten T^j\}$ are independent under the assumption that the data samples are independent. 
By the Hoeffding's inequality, for any constant $t$, it holds that 
\begin{equation} 
  \mathbb{E}[\ten T] \le \frac{1}{m_{\text{PCE}}}\sum_j \ten T^j  + t\ \ \text{ w.p. } 1-\gamma^*,  
\end{equation} 
with $\gamma^*=\exp(-\frac{2t^2m_{\text{PCE}}}{(b-a)^2})$. 
Equivalently, 
\begin{equation}\label{equ:exp_PCE_Bound}
	\mathbb{E}_{(\mat y,\mat z)}[ \|\mat y - \psi^*(\mat z)\|^2 ]\le \delta^* + t\ \ \text{ w.p. } 1-\gamma^*,  
\end{equation} 
Consequently, there is 
\begin{align*}
&\mathbb{E}_{\mat z}[(\text{loss}(\psi^*(\mat z)))^2]\\
\le& 2\mathbb{E}_{\mat y}[(\text{loss}(\mat y))^2] + 2\mathbb{E}_{(\mat y,\mat z)}[(\text{loss}(\psi^*(\mat z))-\text{loss}(\mat y))^2] \\
\le& 2\mathbb{E}_{\mat x_0}[(c_0(\mat x_0))^2]+
2n_L(\delta^*+t)   \ \text{ w.p. } 1-\gamma^*.
\end{align*}
The last inequality follows from $c_0(\mat x_0)=c_l(\mat x_l)=\text{loss}(\mat y)$, equations (\ref{equ:bound_Lip}) and (\ref{equ:exp_PCE_Bound}). 
This completes the proof.
\end{proof}

 Lemma \ref{lem:err_PCE} shows with a high probability $1-\gamma^*$, the expected error of ASNet without fine-tuning is bounded by the pre-trained error of the original network, the accuracy loss in solving the polynomial chaos subproblem (\ref{equ:PCE_optim}), and the number of classes $n_L$.
 The probability $\gamma^*$ is controlled by the threshold $t$ as well as the number of training samples $m_{\text{PCE}}$.  

 In practice, we always re-train ASNet for several epochs and the accuracy of ASNet is beyond the scope of Lemma \ref{lem:err_PCE}. 
}

\subsection{Structured Re-training of ASNet}
\label{subsec:sparse}
 
The  pre-model can be further compressed by various techniques such as network pruning and sharing~\cite{han2015deep}, low-rank factorization~\cite{novikov2015tensorizing, lebedev2014speeding, garipov2016ultimate}, or data quantization~\cite{deng2018gxnor, courbariaux2016binarized}. Denote $\boldsymbol \theta$ as   the weights in ASNet and $\{\mat x_0^1,\ldots,\mat x_0^m\}$ as  the training dataset. 
\ccf{Here, $\boldsymbol \theta$ denotes all the parameters in the pre-model, active subspace layer, and the polynomial chaos expansion layer.} 
We  re-train the network by solving the following regularized optimization problem:
\begin{equation}
\label{eq:sparse_optimization}
\boldsymbol \theta^* =\arg\min_{\boldsymbol \theta}\  \frac1m\sum_{i=1}^m\text{loss}(f(\boldsymbol \theta;\mat x_0^i))+ \lambda R(\boldsymbol \theta). 
\end{equation}
Here $(\mat x_0^i,\mat y^i)$ is a training sample, $m$ is the total  number of training samples, $\text{loss}(\cdot)$ is the  cross-entropy loss function, $R(\boldsymbol \theta)$ is a regularization function, and $\lambda$ is a regularization parameter. 
Different regularization functions can result in different model structures. For instance, an $\ell_1$ regularizer $R(\boldsymbol \theta)=\|\boldsymbol \theta\|_1$ \cite{aghasi2017net,scardapane2017group,ye2019progressive}  will return a sparse weight, an  $\ell_{1,2}$-norm regularizer will result in a column-wise sparse weights, a nuclear norm regularizer will result in   low-rank weights. 
At each iteration, we solve (\ref{eq:sparse_optimization}) by a  stochastic proximal gradient decent algorithm  \cite{shalev2014accelerated} 
\begin{equation}\label{equ:sparseSGD}
\boldsymbol \theta^{k+1} =\underset{\boldsymbol \theta}{\text{argmax}}\quad (\boldsymbol \theta-\boldsymbol \theta^{k})^T\mat g^k+\frac{1}{2\alpha_k}\|\boldsymbol \theta-\boldsymbol \theta^{k}\|_2^2+ \lambda R(\boldsymbol{\theta}). 
\end{equation}
Here   $\mat g^k= \frac{1}{|\mathcal{B}_k|} \sum_{i\in\mathcal{B}_k}\nabla_{\boldsymbol \theta}\text{loss}(f(\boldsymbol \theta;\mat x_0^i), \mat y^i)$ is the stochastic gradient, $\mathcal{B}_k$ is a batch at the $k$-th step,  and   $\alpha_k$ is the stepsize.  

In this work, we chose the $\ell_1$  regularization to get sparse weight matrices. 
In this case, problem (\ref{equ:sparseSGD})   has a closed-form solution:
\begin{equation} 
   \boldsymbol \theta^{k+1} = \mathcal{S}_{\alpha_k\lambda}(\boldsymbol \theta^{k}-\alpha_k \mat g^k), 
   \end{equation} 
  where $\ten S_{\lambda}(\mat  x)=\mat x\odot \max(0,1-\lambda/|\mat  x|)$  is a soft-thresholding operator.

 \section{Active-Subspace for Universal Adversarial Attacks}
 \label{sec:adv}

This section investigates how to generate a universal adversarial attack by the active-subspace method. Given a  function $f(\mat x)$, the maximal perturbation direction is defined by
\begin{equation}\label{equ:AS_per}
    \mat v_{\delta}^*=\underset{\|\mat v\|_2\le\delta}{\text{argmax}} \quad \mathbb E_{\mat x}[(f(\mat x+\mat v)-f(\mat x))^2]. 
\end{equation}
Here, $\delta$ is a user-defined perturbation upper bound. By   the first order Taylor expansion, we have  $f(\mat x+\mat v)\approx f(\mat x) + \nabla f(\mat x)^T\mat v$, and problem (\ref{equ:AS_per}) can be reduced to 
\begin{equation}\label{equ:AS_lin}
    \mat v_{AS}=\underset{\|\mat v\|_2=1}{\text{argmax}} \quad \mathbb E_{\mat x}[(\nabla f(\mat x)^T\mat v)^2] =\underset{\|\mat v\|_2=1}{\text{argmax}} \quad  \mat v^T \mathbb E_{\mat x}[\nabla f(\mat x)\nabla f(\mat x)^T]\mat v. 
\end{equation}
The vector $\mat v_{AS}$ is exactly  the dominant eigenvector of the covariance matrix of $\nabla f(\mat x)$. The  solution for (\ref{equ:AS_per}) can be approximated by $+\delta\mat v_{AS}$ or  $-\delta\mat v_{AS}$. Here,   both $\mat v_{AS}$ and $-\mat v_{AS}$ are solutions of  (\ref{equ:AS_lin}) but their effect on (\ref{equ:AS_per}) are different.

 \begin{figure}[t]
    \centering
    \includegraphics[width=\textwidth]{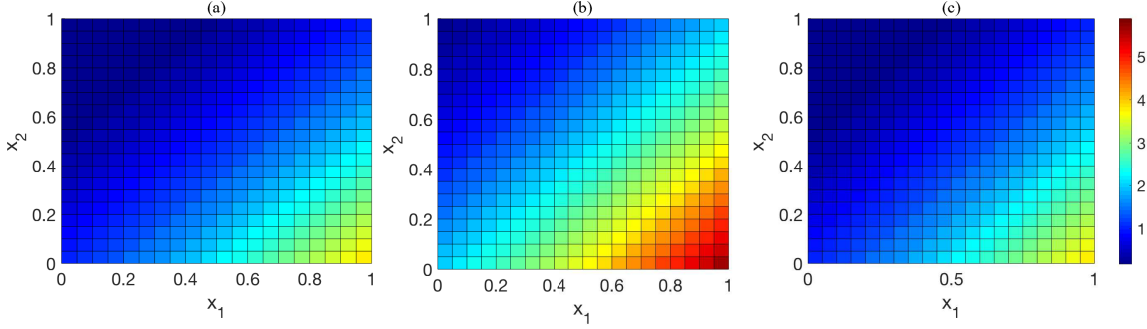}
    \caption{\small Perturbations along the directions of an active-subspace direction and  of principal component, respectively. (a) The function $f(\mat x)=\mat a^T\mat x-b$. (b) The perturbed function along the active-subspace direction. (c) The perturbed function along the principal component analysis direction.
    }
    \label{fig:AS_illustrate}
\end{figure}

\begin{example2}
Consider a two-dimensional function $f(\mat x)=\mat a^T\mat x-b$ with $\mat a=[1,-1]^T$ and $b=1$, \ccf{and $\mat x$ follows a uniform distribution} in a two-dimensional square domain $[0,1]^2$, as shown in Fig.~\ref{fig:AS_illustrate}~(a). 
It follows from direct computations that $\nabla f(\mat x)=\mat a$ and the covariance matrix $\mat C =  \mat a\mat a^T$. The  dominant  eigenvector of $\mat C$ or  the active-subspace direction is $\mat v_{AS}=\mat a/\|\mat a\|_2=[1/\sqrt{2},-1/\sqrt{2}]$.  
We apply $\mat v_{AS}$ to perturb $f(\mat x)$ and plot    $f(\mat x+\delta \mat v_{AS})$ in  Fig.~\ref{fig:AS_illustrate}~(b), which shows a significant difference even for  a small permutation  $\delta=0.3$.  
Furthermore, we plot the perturbed function along the 
 first principal component direction  $\mat w_1=[1/\sqrt{2},1/\sqrt{2}]^T$ in Fig.~\ref{fig:AS_illustrate}~(c). 
 \ccf{Here, $\mat w_1$ is the eigenvector of the covariance matrix $\mathbb E_{\mat x}[\mat x\mat x^T]=\left[\begin{array}{cc}
1/3 & 1/4\\
1/4 & 1/3
 \end{array}\right]$.}
 However, $\mat w_1$ does not result in any perturbation because $\mat a^T\mat w_1=0$. 
This example indicates the difference between the active-subspace and  principal component analysis: the active-subspace direction can  capture the sensitivity information of $f(\mat x)$ whereas the principal component is independent of  $f(\mat x)$.  
\end{example2}

\subsection{Universal Perturbation of Deep Neural Networks}

Given a dataset $\ten D$ and a classification function  $j(\mat x)$  that maps an input sample to an output label.  The universal perturbation seeks for a vector $\mat v^*$  whose norm is upper bounded by $\delta$,  such that the class label can be perturbed with a high probability, i.e., 
\begin{equation}\label{equ:UAP}
  \mat v^*=\underset{\|\mat v\|\le\delta}{\text{argmax}}\  \text{prob}_{\mat x\in\ten D}[j(\mat x+\mat v) \neq j(\mat x)]=\underset{\|\mat v\|\le\delta}{\rm argmax}\ \mathbb{E}_{\mat x}[1_{j(\mat x+\mat v) \neq j(\mat x)}],
\end{equation}
where $1_d$ equals  one if the condition $d$ is satisfied and zero otherwise.  
Solving  problem (\ref{equ:UAP}) directly is challenging because both $1_d$ and $j(\mat x)$ are discontinuous. 
By replacing $j(\mat x)$ with the    loss function $c(\mat x)=\text{loss}(f(\mat x))$ and the indicator function $1_d$ with a   quadratic function,  we reformulate problem (\ref{equ:UAP}) as 
\begin{equation}\label{equ:UAP_L2}
    \max_{\mat v}\quad \mathbb{E}_{\mat x}[\left(c(\mat x+\mat v)-c(\mat x)\right)^2] \quad \text{s.t.}\quad \|\mat v\|_2\le \delta.
\end{equation} 

 The ball-constrained optimization problem  (\ref{equ:UAP_L2}) can be solved by various numerical techniques such as the spectral gradient descent method \cite{birgin2000nonmonotone} and the limited-memory projected quasi-Newton \cite{schmidt2009optimizing}. However, these methods can only guarantee convergence to a local stationary point. Instead, we are interested in computing a direction that can achieve a better objective value by a heuristic algorithm.
 

\subsection{Recursive Projection Method}

Using the first order Taylor expansion    $c(\mat x+\mat v)\approx c(\mat x) + \mat v^T\nabla c(\mat x)$, we reformulate problem (\ref{equ:UAP_L2})   as a ball constrained  quadratic problem 
\begin{equation}\label{equ:UAP_eig}
    \max_{\mat v}\quad \mat v^T \mathbb{E}_{\mat x}[\nabla c(\mat x) \nabla c(\mat x)^T]\mat v \quad \text{s.t.}\quad \|\mat v\|_2\le \delta.  
\end{equation} 
Problem (\ref{equ:UAP_eig}) is easy to solve because its closed-form solution is exactly  the dominant eigenvector of the covariance  matrix $\mat C=\mathbb{E}_{\mat x}[\nabla c(\mat x) \nabla c(\mat x)^T]$ or the first active-subspace direction.  
However,  the   dominant  eigenvector in  (\ref{equ:UAP_eig}) may not  be  efficient because $c(\mat x)$ is nonlinear.  
Therefore, we compute $\mat v$   recursively  by 
\begin{equation}\label{equ:proj}
    \mat v^{k+1}=\text{proj}(\mat v^k+s^k \ccf{d_{\mat v}^k}),
\end{equation}
where $\text{proj}(\mat v) = \mat v \times \min(1,  \delta/\|\mat v\|_2)$, 
  $s^k$ is the stepsize, and $\ccf{d_{\mat v}^k}$ is approximated by 
\begin{equation}\label{equ:UAP_eig_i}
   \ccf{d_{\mat v}^k}=\underset{\ccf{d_{\mat v}}}{\text{argmax}}\quad \ccf{d_{\mat v}^T} \mathbb{E}_{\mat x}\left[\nabla c\left(\mat x+\mat v^k\right) \nabla c\left(\mat x+\mat v^k\right)^T\right]\ccf{d_{\mat v}}, \ \text{s.t.}\ \|d_{\mat v}\|_2\le 1.
\end{equation}
Namely, $\ccf{d_{\mat v}^k}$ is the dominant  eigenvector of  $\mat C^k=\mathbb{E}_{\mat x}\left[\nabla c\left(\mat x+\mat v^k\right) \nabla c\left(\mat x+\mat v^k\right)^T\right]$. 
Because $\ccf{d_{\mat v}^k}$ maximizes the changes in $\mathbb E_{\mat x}[(c(\mat x+\mat v+\ccf{d_{\mat v}}) - c(\mat x+\mat v))^2]$, we expect that the attack ratio keeps increasing, i.e.,   $r( \mat v^{k+1}; \ten D)\ge r( \mat v^k; \ten D)$, where  
\begin{equation}
    \label{equ:attratio}
    r( \mat v; \ten D) = \frac{1}{|\ten D|}\sum_{\mat x^i\in\ten D}1_{j(\mat x^i+\mat v)\neq j(\mat x^i)}. 
\end{equation} 
The backtracking line search  approach \cite{armijo1966minimization} is employed to choose $s^k$ such that the attack ratio of  $\mat v^k+s^k\ccf{d_{\mat v}^k}$  is higher than the attack ratio of both   $\mat v^k$ and  $\mat v^k-s^k\ccf{d_{\mat v}^k}$, i.e., 
\begin{equation}\label{equ:sz}
 s^k=\min_i \{ s^k_{i,t}:  r(\mat v^{k+1}_{i,t}; \ten D)> \max(r(\mat v^{k+1}_{i,-t}; \ten D) , r(\mat v^k; \ten D)\}, 
\end{equation}
where   $s^k_{i,t}=(-1)^ts_0\gamma^i$, $t\in\{1,-1\}$, $s_0$ is the initial stepsize, $\gamma <1$ is the decrease ratio, and $\mat v^{k+1}_{i,t}=\text{proj}(\mat v^k+s^{k+1}_{i,t}\ccf{d_{\mat v}^k})$. 
If such a stepsize $s^k$ exists, we update $\mat v^{k+1}$ by (\ref{equ:proj}) and repeat the process. Otherwise, we record the number of failures and stop the algorithm when the number of failure is greater than a threshold. 

The overall flow is summarized in Algorithm~\ref{alg:attack_UAS}. 
\ccf{In practice, instead of using the whole dataset to train this attack vector, we use a subset $\ten D^0$. 
The impact for different  number of samples is discussed  in section~\ref{sub:adversal_CIFAR10}.}

\begin{algorithm}[t]
\caption{Recursive Active Subspace Universal Attack}
\label{alg:attack_UAS}
\textbf{Input: } {A pre-trained deep neural network  denoted as $c(\mat x)$, a classification oracle $j(\mat x)$, a   training dataset $\ten D^0$, an  upper bound  for  the attack vector $\delta$, an initial stepsize  $s_0$,  a decrease ratio $\gamma <1$, and the parameter in the stopping criterion $\alpha$.}

\begin{algorithmic}[1]
 \STATE Initialize the attack vector as $\mat v^0=0$.
 
\FOR {$k=0,1,
\ldots$} 
  \STATE  
    Select the training dataset as $\ten D=\{ \mat x^i+\mat v^k:\ \mat x^i\in\ten D^0 \text{ and } j(\mat x^i+\mat v^k)=j(\mat x^i)\}$, then compute the dominate active subspace direction $d\mat v$ by  Algorithm~\ref{alg:onlineASMat}.  
    
  \FOR {$i=0,1,...I$}
    \STATE Let $s^k_{i,\pm}=(-1)^{\pm}s_0\gamma^i$  and $\mat v^{k+1}_{i,\pm}=\text{proj}(\mat v^k+s^{k+1}_{i,\pm}\ccf{d_{\mat v}^k})$ . Compute the attack ratios $r(\mat v^{k+1}_{i,1})$ and $r(\mat v^{k+1}_{i,-1})$ by (\ref{equ:attratio}). 
    \STATE If either $r(\mat v^{k+1}_{i,1})$ or $r(\mat v^{k+1}_{i,-1})$ is greater than $r(\mat v^k)$, stop the process. Return $s^k=(-1)^ts^k_{i,1}$, where $t=1$ if  $r(\mat v^{k+1}_{i,1})\ge r(\mat v^{k+1}_{i,-1})$ and $t=-1$ otherwise.  
  \ENDFOR
  
  If no stepsize $s^k$ is returned, let $s^k=s_0r^I$  and record this step as a failure. 
 Compute  the next iteration $\mat v^{k+1}$ by the projection  (\ref{equ:proj}).  

   \STATE    If the number of failure is greater the threshold \ccf{$\alpha$}, stop.  
 \ENDFOR
 
\textbf{Output:} {The universal active adversarial  attack vector $\mat v_{AS}$.}
\end{algorithmic}
\end{algorithm}

\section{Numerical Experiments}
\label{sec:Numerical}

In this section, we show the power of active-subspace in revealing the number of active neurons, compressing neural networks, and computing the universal  adversarial perturbation. 
All codes are implemented  in  PyTorch and are available  online\footnote{\url{https://github.com/chunfengc/ASNet}}. 

\subsection{Structural Analysis and Compression}

We test the  ASNet constructed by Algorithm~\ref{alg:ASNet}, and set the polynomial order as  $p=2$,  the number of active neurons   as $r=50$, and the  threshold  in Equation \eqref{equ:nAS} as  $\epsilon=0.05$ \ccf{on default}. 
\ccf{Inspired by the knowledge distillation~\cite{hinton2015distilling}, we retrain all the parameters in the ASNet by minimizing the following loss function
\begin{equation*}
\min_{\boldsymbol \theta}\ \sum_{i=1}^m \beta H\left(\text{ASNet}_{\boldsymbol \theta}(\mat x_0^i),f(\mat x_0^i)\right) +  (1-\beta) H\left(\text{ASNet}_{\boldsymbol \theta}(\mat x_0^i), \mat y^i\right). 
\end{equation*}
Here, the cross entropy $H(\mat p, \mat q)=\sum_j s(\mat p)_j\log s(\mat q)_j$, the softmax function $s(\mat x)_j=\frac{\exp(x_j)}{\sum_j \exp(x_j)}$,  and the parameter  $\beta=0.1$ on default. 
We retrain ASNet for 50 epochs by ADAM  \cite{kingma2014adam}. 
The stepsizes for the pre-model are set as $10^{-4}$ and $10^{-3}$ for VGG-19 and  ResNet,  
and the stepsize for the active subspace layer and the polynomial chaos expansion layer is set as $10^{-5}$,   respectively,  

We also  seek for   sparser weights in ASNet by the proximal stochastic gradient descent  method in Section~\ref{subsec:sparse}.  
On default, we set the stepsize as $10^{-4}$ for the pre-model and $10^{-5}$ for the active subspace layer and the polynomial chaos expansion layer. The maximal epoch is set as 100.  
The obtained sparse model is denoted as ASNet-s. 

In all figures and tables, the numbers in the bracket of ASNet($\cdot$) or ASNet-s($\cdot$) indicate the index of a cut-off layer. 
 We report the performance  for different cut-off layers in terms of {\it accuracy, storage, and computational complexities}.
}


\ccf{
\subsubsection{Choices of Parameters} 
We first show the influence of number of reduced neurons $r$,  tolerance $\epsilon$, and cutting-off layer index $l$ of VGG-19 on CIFAR-10 in   Table \ref{tab:comp_param}.  The VGG-19 can achieve 93.28\% testing accuracy with 76.45 Mb stroage consumption.  
Here, $\epsilon=\frac{\lambda_{r+1}+\ldots+\lambda_n}{\lambda_1+\ldots+\lambda_n}$. 
For different choices of $r$, 
we display the corresponding tolerance $\epsilon$,  the storage speedup compared with the original teacher network, and the testing  accuracy reduction for  ASNet before and after fine-tuning compared with the original teacher network.  

 Table \ref{tab:comp_param} shows that   when the cutting-off layer is fixed, a larger $r$ usually results in a smaller tolerance $\epsilon$ and  a smaller  accuracy reduction but also a smaller storage speedup. 
  This is corresponding to Lemma \ref{lem: ASerror} that the error of ASNet before fine-tuning is upper bounded by $O(\epsilon)$.  
Comparing   $r=50$ with $r=75$, we find that $r=50$ can achieve almost the same accuracy with $r=75$ with a higher storage speedup. 
$r=50$ can even achieve better accuracy than $r=75$ in layer 7   probably because of overfitting. This guides us to chose $r=50$ in the following numerical experiments. 
For different  layers, we see a later cutting-off layer index can  produce a lower accuracy reduction but a smaller storage speedup. 
In other words, the choice of layer index is a trade-off between accuracy reduction with storage speedup. 
}

{\color{red}

\begin{table}[t]
 \caption{\ccf{Comparison of number of neurons $r$ of VGG-19 on CIFAR-10. For the stroage speedup, the higher is bettter. For the accuracy reduction before or after finetuning, the lower is better.}}
  \label{tab:comp_param}
    \centering
    {\small
    \begin{tabular}{|@{\hspace{0.1cm}}c@{\hspace{0.1cm}}
    |@{\hspace{0.1cm}}c@{\hspace{0.1cm}}c@{\hspace{0.02cm}}c@{\hspace{0.1cm}}c@{\hspace{0.1cm}}
    |@{\hspace{0.1cm}}c@{\hspace{0.1cm}}c@{\hspace{0.02cm}}c@{\hspace{0.1cm}}c@{\hspace{0.1cm}}
    |@{\hspace{0.1cm}}c@{\hspace{0.1cm}}c@{\hspace{0.02cm}}c@{\hspace{0.1cm}}c@{\hspace{0.1cm}}|}
    \hline
 &  \multicolumn{4}{c|@{\hspace{0.1cm}}}{$r=25$} & \multicolumn{4}{c|@{\hspace{0.1cm}}}{$r=50$} & 
 \multicolumn{4}{c|}{$r=75$}\\ \hline
 & $\epsilon$ & Storage & \multicolumn{2}{c|@{\hspace{0.1cm}}}{\footnotesize Accu. Reduce} 
 & $\epsilon$ & Storage &  \multicolumn{2}{c|@{\hspace{0.1cm}}}{\footnotesize Accu. Reduce} 
 & $\epsilon$ & Storage &  \multicolumn{2}{c|}{\footnotesize Accu. Reduce} \\  
 & && {\footnotesize Before} & {\footnotesize After} 
 & & &{\footnotesize Before} & {\footnotesize After} 
 & && {\footnotesize Before} & {\footnotesize After}  \\\hline
ASNet(5) & 0.34 & \textbf{20.7$\times$} &  7.06 &2.82  & 0.18 & 14.4$\times$ & 4.40 & 1.82  &0.11 & 11.0$\times$ &  \textbf{3.64}  & \textbf{1.66} \\\hline 
ASNet(6) & 0.24 & \textbf{12.8$\times$} &  2.14 &0.59  & 0.11 & 10.1$\times$ & 1.62 & 0.27  &0.05 & 8.3$\times$ &  \textbf{1.40}  & \textbf{0.21} \\\hline 
ASNet(7) & 0.15 & \textbf{9.3$\times$} &  0.79 &0.11  & 0.06 & 7.8$\times$ & \textbf{0.63} & \textbf{-0.10}  &0.03 & 6.7$\times$ &  0.77  & 0.00 \\\hline 
    \end{tabular}
    }
\end{table}
}

 \subsubsection{Efficiency of Active-subspace} 
We  show the effectiveness of  ASNet  constructed  by  Steps 1-3 of Algorithm~\ref{alg:ASNet} without fine-tuning. 
We investigate the following three properties. 
(1) {\bf Redundancy of neurons.}  The distributions of the first 200 singular values of the   matrix $\hat{\mat G}$ (defined in \eqref{equ:G}) are plotted in Fig.~\ref{fig:lmd_PCA_LR}~(a). The singular values decrease almost exponentially  for  layers $l\in\{4,5,6,7\}$. 
Although the total numbers of neurons are 8192, 16384, 16384, and 16384,  the numbers of active neurons are only 105,  84,  54, and 36,  respectively. 
(2) {\bf Redundancy of the layers.} We cut off the deep \ccf{neural network} at an intermediate layer and replace the subsequent layers with one simple logistic regression~\cite{hosmer2013applied}. 
As shown by the red bar in Fig.~\ref{fig:lmd_PCA_LR}~(b),  the logistic regression can achieve relatively high accuracy. This verifies that the features trained from the first few layers already have a high expression power since replacing all subsequent layers with a simple expression loses little accuracy.   
(3)~{\bf Efficiency of the active-subspace and polynomial chaos expansion.} 
We compare the proposed active-subspace layer with the principal component analysis~\cite{jolliffe2011principal} in  projecting the high-dimensional neuron to a low-dimensional space, and also compare the polynomial chaos expansion layer with logistic regression in terms of their efficiency to extract class labels from the low-dimensional variables.  Fig.~\ref{fig:lmd_PCA_LR}~(b)  shows that the combination of active-subspace and polynomial chaos expansion can achieve the best accuracy.

\begin{figure}[t]
    \centering
        \includegraphics[width=0.9\textwidth]{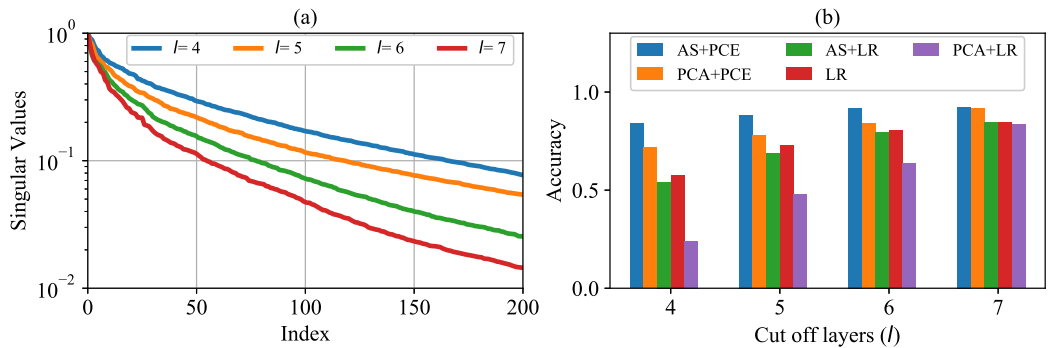}
    \caption{\small Structural analysis of VGG-19 on the CIFAR-10 dataset. (a) The first 200 singular values for layers $4\le l\le7$; (b) The accuracy (without any fine-tuning) obtained by active-subspace (AS) and polynomial chaos expansions (PCE) compared with principal component analysis (PCA) and logistic regression (LR).}
    \label{fig:lmd_PCA_LR}
\end{figure}


\begin{table}[t]
    \centering 
    \setlength{\tabcolsep}{5pt}
   \footnotesize{   
        \caption{\small Accuracy and storage on VGG-19 for CIFAR-10. 
        \ccf{Here, ``Pre-M" denotes the pre-model, i.e., layers 1 to $l$ of the original deep neural networks, ``AS" and ``PCE" denote the  active subspace and polynomial chaos expansion layer, respectively. }}
    \label{tab:cifar10_vgg}

\begin{tabular}{|c|c|ccc|ccc|}
    \hline
    Network  & Accuracy & \multicolumn{3}{c|}{Storage  (MB)}& \multicolumn{3}{c|}{Flops ($10^6$)}\\
    \hline
      VGG-19   &  93.28\% &\multicolumn{3}{c|}{76.45}& \multicolumn{3}{c|}{398.14} \\ \hline 
 && Pre-M & AS+PCE & Overall& Pre-M & AS+PCE & Overall \\  \hline
 ASNet(5) &  91.46\% & 2.12 & 3.18 & 5.30& 115.02& 0.83& 115.85 \\
 &  & & (23.41$\times$) & (14.43$\times$) & & (340.11$\times$) & (3.44$\times$) \\\hline
ASNet-s(5) &  90.40\% & 1.14 & 2.05 &  \ccf{\textbf 3.19}& 54.03& 0.54&  \ccf{\textbf 54.56} \\
 &  & (1.86$\times$)& (36.33$\times$) & \ccf{\textbf  (23.98$\times$)} & (2.13$\times$)& (527.91$\times$) &  \ccf{\textbf (7.30$\times$)} \\\hline
ASNet(6) &  93.01\% & 4.38 & 3.18 & 7.55& 152.76& 0.83& 153.60 \\
 &  & & (22.70$\times$) & (10.12$\times$) & & (294.76$\times$) & (2.59$\times$) \\\hline
ASNet-s(6) &  91.08\% & 1.96 & 1.81 & 3.77& 67.37& 0.48& 67.85 \\
 &  & (2.24$\times$)& (39.73$\times$) & (20.27$\times$) & (2.27$\times$)& (515.98$\times$) & (5.87$\times$) \\\hline
ASNet(7) &   \ccf{\textbf 93.38\%} & 6.63 & 3.18 & 9.80& 190.51& 0.83& 191.35 \\
 &  & & (21.99$\times$) & (7.80$\times$) & & (249.41$\times$) & (2.08$\times$) \\\hline
ASNet-s(7) &  90.87\% & 2.61 & 1.91 & 4.52& 80.23& 0.50& 80.73 \\
 &  & (2.54$\times$)& (36.64$\times$) & (16.92$\times$) & (2.37$\times$)& (415.68$\times$) & (4.93$\times$) \\\hline
    \end{tabular}
}
\end{table}
\begin{table}[t]
\setlength{\tabcolsep}{5pt}
    \centering
\footnotesize{\caption{Accuracy and storage on ResNet-110 for CIFAR-10. 
\ccf{Here, ``Pre-M" denotes the pre-model, i.e., layers 1 to $l$ of the original deep neural networks, ``AS" and ``PCE" denote the  active subspace and polynomial chaos expansion layer, respectively. }}
  
     \label{tab:cifar10_resnet}    
  
    \begin{tabular}{|c|c|ccc|ccc|}
    \hline
        Network  & Accuracy & \multicolumn{3}{c|}{Storage  (MB)}& \multicolumn{3}{c|}{Flops ($10^6$)}\\
    \hline
      ResNet-110   &  93.78\% &\multicolumn{3}{c|}{6.59}& \multicolumn{3}{c|}{252.89} \\ \hline 
       && Pre-M & AS+PCE & Overall&Pre-M & AS+PCE & Overall \\  \hline 
      ASNet(61) &  89.56\% & 1.15 & 1.61 & 2.77& 140.82& 0.42& 141.24 \\
 &  & & (3.37$\times$) & (2.38$\times$) & & (265.03$\times$) & (1.79$\times$) \\\hline
ASNet-s(61) &  89.26\% & 0.83 & 1.23 &  \ccf{\textbf 2.06}& 104.05& 0.32&  \ccf{\textbf 104.37} \\
 &  & (1.39$\times$)& (4.41$\times$) &  \ccf{\textbf (3.19$\times$)} & (1.35$\times$)& (346.82$\times$) &  \ccf{\textbf (2.42$\times$)} \\\hline
ASNet(67) &  90.16\% & 1.37 & 1.61 & 2.98& 154.98& 0.42& 155.40 \\
 &  & & (3.24$\times$) & (2.21$\times$) & & (231.55$\times$) & (1.63$\times$) \\\hline
ASNet-s(67) &  89.69\% & 1.00 & 1.22 & 2.22& 116.38& 0.32& 116.70 \\
 &  & (1.36$\times$)& (4.29$\times$) & (2.97$\times$) & (1.33$\times$)& (306.72$\times$) & (2.17$\times$) \\\hline
ASNet(73) &   \ccf{\textbf 90.48\%} & 1.58 & 1.61 & 3.19& 169.13& 0.42& 169.55 \\
 &  & & (3.11$\times$) & (2.06$\times$) & & (198.07$\times$) & (1.49$\times$) \\\hline
ASNet-s(73) &  90.02\% & 1.18 & 1.16 & 2.34& 128.65& 0.30& 128.96 \\
 &  & (1.34$\times$)& (4.32$\times$) & (2.82$\times$) & (1.31$\times$)& (275.74$\times$) & (1.96$\times$) \\\hline
    \end{tabular}
}
   
\end{table}

\subsubsection{CIFAR-10}
We continue to present the   results of ASNet and  ASNet-s on CIFAR-10 by  two widely used networks:  VGG-19 and ResNet-110 in   
 Tables~\ref{tab:cifar10_vgg}  and \ref{tab:cifar10_resnet}, respectively. 
 The second column shows the testing accuracy for the corresponding network. We report the storage and computational costs for the pre-model, post-model (i.e., active-subspace plus polynomial chaos expansion for ASNet and ASNet-s), and overall results,  respectively. 
 For both examples, ASNet and ASNet-s can achieve  a similar  accuracy with the teacher  network yet with  much smaller storage and computational cost.  
 For VGG-19, ASNet achieves  $14.43\times$ storage savings and $3.44\times$ computational reduction; ASNet-s achieves $23.98\times$ storage savings and $7.30\times$ computational reduction.  For most ASNet and ASNet-s networks, the storage and computational costs of the post-models achieve significant performance boosts  by our proposed network structure changes.  
It is not  surprising to see that    increasing the layer index (i.e., cutting off the    deep neural network at a later layer) can produce a  higher accuracy. 
\ccf{However, increasing the layer index also results in a smaller compression ratio. 
In other words, the choice of layer index is a trade-off between the accuracy reduction with the compression ratio. 
}
 
\ccf{For Resnet-110, our results are not as good as those on VGG-19. We find that the eigenvalues  for its covariance matrix are not exponentially decreasing as that of VGG-19, which results in a large number of active neurons or a large error $\epsilon$ when fixing $r=50$. A possible reason is that   ResNet updates as $\mat x_{l+1}=\mat x_l+f_l(\mat x_l)$. 
Hence, the partial gradient $\partial \mat x_{l+1}/\partial \mat x_l = I + \nabla f_l(\mat x_l)$ is less likely to be low-rank.}


\begin{table}[t]
    \centering
    \setlength{\tabcolsep}{5pt}
\footnotesize{        \caption{\small Accuracy and storage on VGG-19 for CIFAR-100. 
\ccf{Here, ``Pre-M" denotes the pre-model, i.e., layers 1 to $l$ of the original deep neural networks, ``AS" and ``PCE" denote the  active subspace and polynomial chaos expansion layer, respectively. }}
    \label{tab:cifar100_vgg}
    \begin{tabular}{|c|cc|ccc|ccc|}
    \hline
    Network  & Top-1 & Top-5 & \multicolumn{3}{c|}{Storage  (MB)}& \multicolumn{3}{c|}{Flops ($10^6$)}\\
    \hline
      VGG-19   & 71.90\%& 89.57\%  &\multicolumn{3}{c|}{76.62}& \multicolumn{3}{c|}{398.18} \\ \hline 
 &&& Pre-M & AS+PCE & Overall&Pre-M & AS+PCE & Overall \\  \hline
ASNet(7) &  70.77\% &  91.05\% &6.63 & 3.63 & 10.26& 190.51& 0.83& 191.35 \\
 &  & & & (19.23$\times$) & (7.45$\times$) & & (249.41$\times$) & (2.08$\times$) \\\hline
ASNet-s(7)   &70.20\% & 90.90\% & 5.20 & 3.24& 8.44& 144.81& 0.85&  \ccf{\textbf 145.66} \\
 &  & & (1.27$\times$) & (21.56$\times$) & (9.06$\times$) & (1.32$\times$)& (244.57$\times$) &  \ccf{\textbf (2.73$\times$)} \\\hline
ASNet(8) &  69.50\% &  90.15\% &8.88 & 1.29 & 10.17& 228.26& 0.22& 228.48 \\
 &  & & & (52.50$\times$) & (7.52$\times$) & & (779.04$\times$) & (1.74$\times$) \\\hline
ASNet-s(8)   &69.17\% & 89.73\% & 6.87 & 1.22&  \ccf{\textbf 8.09}& 172.69& 0.32& 173.01 \\
 &  & & (1.29$\times$) & (55.36$\times$) &  \ccf{\textbf (9.45$\times$)} & (1.32$\times$)& (530.92$\times$) & (2.30$\times$) \\\hline
ASNet(9) &   \ccf{\textbf 72.00\%} &   \ccf{\textbf 90.61\%} &13.39 & 2.07 & 15.46& 247.14& 0.42& 247.56 \\
 &  & & & (30.49$\times$) & (4.95$\times$) & & (357.10$\times$) & (1.61$\times$) \\\hline
ASNet-s(9)   &71.38\% & 90.28\% & 9.38 & 1.94& 11.32& 183.27& 0.51& 183.78 \\
 &  & & (1.43$\times$) & (32.49$\times$) & (6.75$\times$) & (1.35$\times$)& (296.74$\times$) & (2.17$\times$) \\\hline
    \end{tabular}
}
\end{table}

\begin{table}[t]
    \centering
    \setlength{\tabcolsep}{5pt}
 \footnotesize{  \caption{Accuracy and storage on ResNet-110 for CIFAR-100. 
 \ccf{Here, ``Pre-M" denotes the pre-model, i.e., layers 1 to $l$ of the original deep neural networks, ``AS" and ``PCE" denote the  active subspace and polynomial chaos expansion layer, respectively. }}
    \label{tab:cifar100_resnet}

 \begin{tabular}{|c|c c|ccc|ccc|}
    \hline
    Network  & Top-1 & Top-5 & \multicolumn{3}{c|}{Storage  (MB)}& \multicolumn{3}{c|}{Flops ($10^6$)}\\
    \hline
   
    ResNet-110 & 71.94\% & 91.71 \% &\multicolumn{3}{c|}{6.61}& \multicolumn{3}{c|}{252.89} \\ \hline 
     &&& Pre-M & AS+PCE & Overall&Pre-M & AS+PCE & Overall \\  \hline  
ASNet(75) &  63.01\% &  88.55\% &1.79 & 1.29 & 3.08& 172.67& 0.22& 172.89 \\
 & & & & (3.73$\times$) & (2.14$\times$) & & (367.88$\times$) & (1.46$\times$) \\\hline
ASNet-s(75)   &63.16\% & 88.65\% & 1.47 & 1.20&  \ccf{\textbf 2.67}& 143.11& 0.31&  \ccf{\textbf 143.42} \\
 &  & & (1.22$\times$) & (3.99$\times$) &  \ccf{\textbf (2.46$\times$)} & (1.21$\times$)& (254.69$\times$) &  \ccf{\textbf (1.76$\times$)} \\\hline
ASNet(81) &  65.82\% &  90.02\% &2.64 & 1.29 & 3.93& 186.83& 0.22& 187.04 \\
 & & & & (3.07$\times$) & (1.68$\times$) & & (302.96$\times$) & (1.35$\times$) \\\hline
ASNet-s(81)   &65.73\% & 89.95\% & 2.20 & 1.21& 3.41& 155.61& 0.32& 155.93 \\
 &  & & (1.20$\times$) & (3.27$\times$) & (1.93$\times$) & (1.20$\times$)& (208.38$\times$) & (1.62$\times$) \\\hline
ASNet(87) &   \ccf{\textbf 67.71\% }&  \ccf{\textbf 90.17\%} &3.48 & 1.29 & 4.77& 200.98& 0.22& 201.20 \\
 & & & & (2.41$\times$) & (1.38$\times$) & & (238.04$\times$) & (1.26$\times$) \\\hline
ASNet-s(87)   &67.65\% & 90.10\% & 2.91 & 1.21& 4.12& 166.50& 0.32& 166.81 \\
 &  & & (1.20$\times$) & (2.56$\times$) & (1.60$\times$) & (1.21$\times$)& (163.50$\times$) & (1.52$\times$) \\\hline
    \end{tabular} 
    }
\end{table}

\subsubsection{CIFAR-100}
Next, we present the results of VGG-19 and ResNet-110 on  CIFAR-100 in  Tables~\ref{tab:cifar100_vgg} and \ref{tab:cifar100_resnet}, respectively. On VGG-19, ASNet can achieve $7.45\times$ storage savings and $2.08\times$  computational reduction, and ASNet-s can achieve $9.06\times$ storage savings and $2.73\times$ computational reduction.  The accuracy loss  is negligible for VGG-19 but larger for ResNet-110. The performance boost of ASNet is obtained by just changing the network structures and without any model compression (e.g., pruning, quantization, or low-rank factorization). 


\subsection{Universal Adversarial Attacks}
This subsection demonstrates the effectiveness of active-subspace in identifying a universal adversarial attack vector. 
We denote the result  generated by  Algorithm~\ref{alg:attack_UAS}  as ``AS'' and compare it with the ``UAP'' method in \cite{moosavi2017universal} and  with ``random" Gaussian distribution vector. 
The parameters in Algorithm~\ref{alg:attack_UAS} are set as $\alpha=10$ and $\delta=5,\ldots,10$.  
The default parameters of UAP are applied except for  the maximal iteration. In the  implementation of~\cite{moosavi2017universal}, the maximal iteration is set as infinity, which is  time-consuming when the training dataset or the number of classes is large. In our experiments, we set the maximal iteration  as 10. 
In all figures and tables,  we report the average attack ratio and CPU time in training out of ten repeated experiments  with different training datasets.  
A higher attack ratio  means the corresponding algorithm is better in fooling the given deep neural network. 
The datasets are chosen in two ways. 
We firstly test  data points from one class (e.g., trousers in Fashion-MNIST) because these data points share lots of common features and have a higher probability to be attacked by a universal perturbation vector.  
We then conduct experiments on the whole dataset to show our proposed algorithm can also provide \ccf{better performance compared with the baseline} even if the dataset has diverse features.

\begin{figure}[ht]
    \centering
    \includegraphics[width=\textwidth]{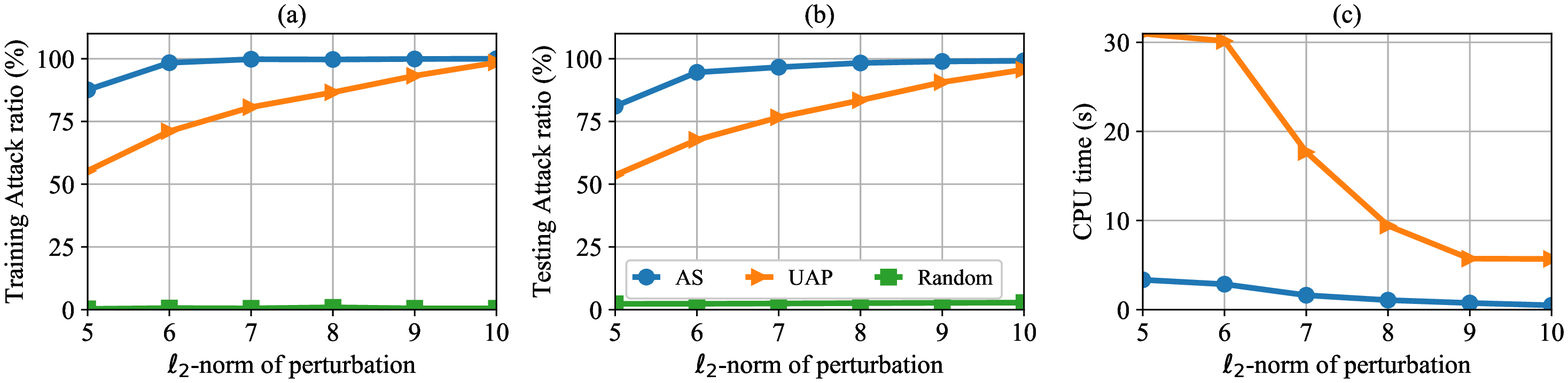}
     \includegraphics[width=\textwidth]{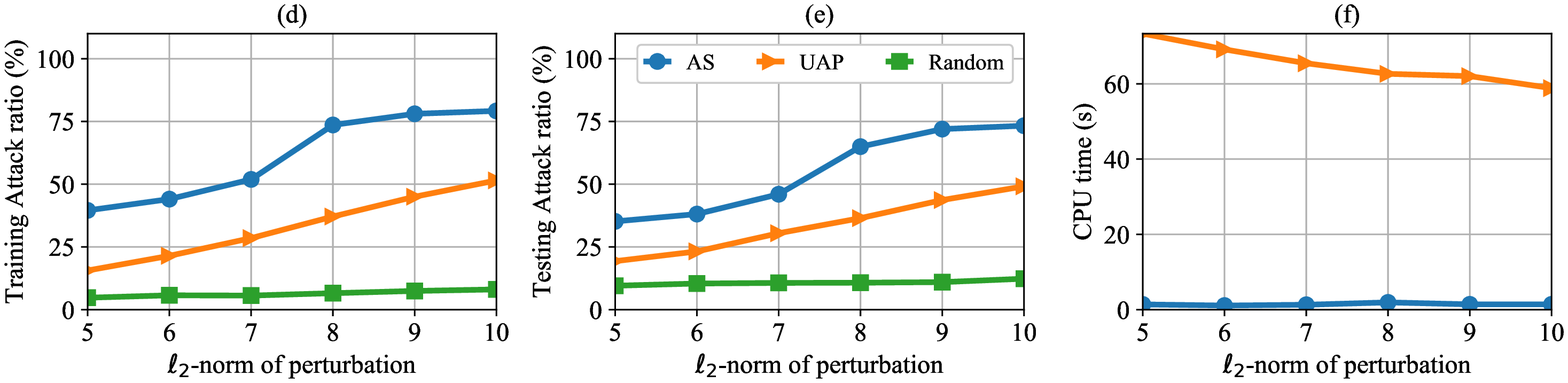}
    \caption{\small Universal adversarial attacks for the  Fashion-MINST with respect to different $\ell_2$-norms. (a)-(c): the results for attacking one class dataset. (d)-(f): the results for attacking the whole dataset. } 
    \label{fig:mnist_attack}
\end{figure}

\subsubsection{Fashion-MNIST}
Firstly, we present the adversarial attack    result on  Fashion-MNIST by a 4-layer neural network. There are two    convolutional layers with kernel size equals  5$\times$5. The size of output channels for each convolutional layer is 20 and 50, respectively. Each convolutional layer is followed by a ReLU activation layer and a max-pooling layer with a kernel size of $2\times 2$.   There are two fully connected layers. The first fully connected layer has an input feature 800 and an output feature 500. 
 
 Fig.~\ref{fig:mnist_attack} presents the attack ratio  of our active-subspace method compared with the baselines UAP method~\cite{moosavi2017universal} and Gaussian random vectors.  The top figures show the results for just one class (i.e., trouser),  and the bottom figures show the results for all ten classes. 
For all perturbation norms, the active-subspace method can achieve around 30\% higher attack ratio  than UAP  while more than 10 times faster. 
This verifies   that the  active-subspace method has better universal representation ability  compared with UAP because the active-subspace can find a universal direction while  UAP   solves   data-dependent subproblems independently.  
By the active-subspace approach,  \ccf{the attack ratio for the first class and the whole dataset are} around 100\% and 75\%, respectively. 
This coincides with our intuition that the  data points in one class have higher similarity than data points from different classes.

In Fig.~\ref{fig:mnist_attack_ex}, we plot one image  from  Fashion-MNIST  and its perturbation  by the active-subspace attack vector. The attacked image in Fig.~\ref{fig:mnist_attack_ex}~(c) still looks like a trouser for a human. However, the  deep neural network misclassifies it as a t-shirt/top.

 
\begin{figure}[t]
    \centering
    \includegraphics[width=\textwidth]{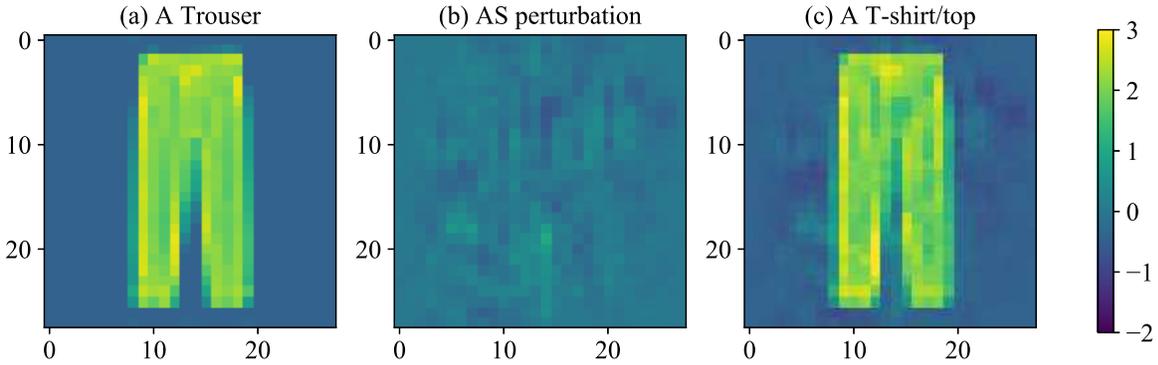}
    \caption{\small The effect of our attack method on one data sample in the Fashion-MNIST dataset. (a) A trouser from the original dataset. (b) An active-subspace perturbation vector with the $\ell_2$ norm equals to 5. (c) The perturbed sample is misclassified as a t-shirt/top by the deep neural network.}
    \label{fig:mnist_attack_ex}
\end{figure}

\begin{figure}[t]
    \centering
    \includegraphics[width=\textwidth]{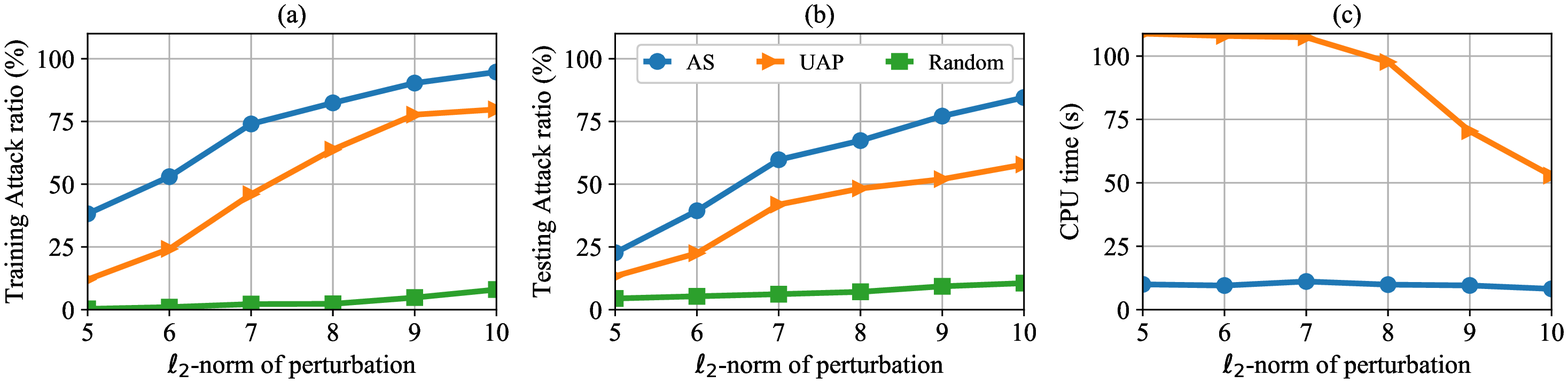}
    \includegraphics[width=\textwidth]{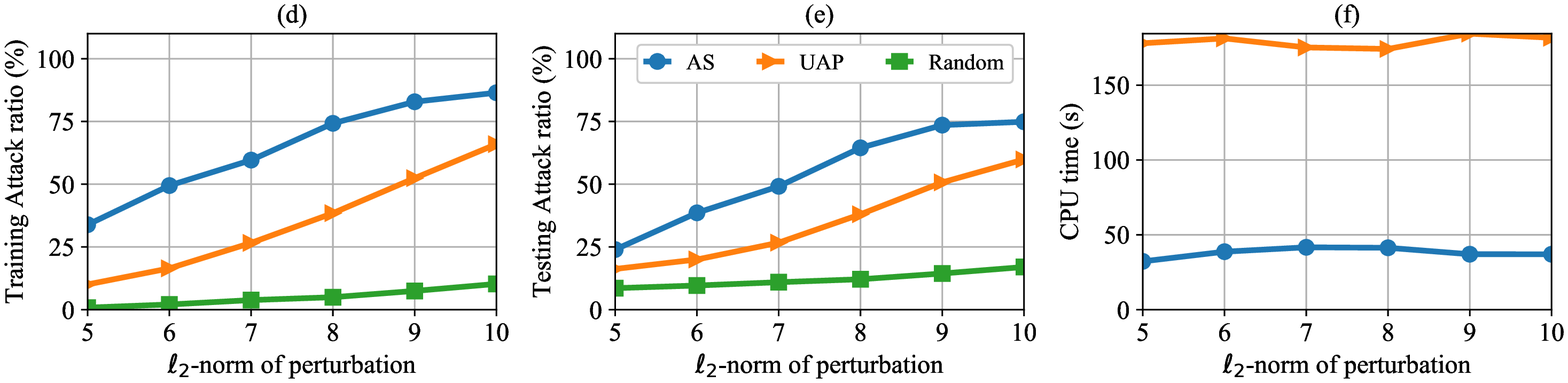}
    \caption{\small Universal adversarial attacks of VGG-19 on CIFAR-10 with respect to different $\ell_2$-norm perturbations. (a)-(c): The training attack ratio, the testing attack ratio, and the  CPU time in seconds for attacking one class dataset. (d)-(f): The results for attacking ten classes dataset together.}
    \label{fig:cifar_attack}
\end{figure}

\subsubsection{CIFAR-10}
\label{sub:adversal_CIFAR10}
Next, we show the numerical results of attacking VGG-19 on CIFAR-10. Fig.~\ref{fig:cifar_attack} compares the active-subspace method compared with the baseline UAP and Gaussian random vectors. 
The top figures show the results by the dataset in the first class  (i.e., automobile), and the bottom figures show the results for all ten classes. 
For both two cases, the proposed active-subspace attack can achieve 20\% higher attack ratios while  three times faster than UAP. 
This is similar to the results in Fashion-MNIST because the active-subspace has a better ability to capture the global information.

\begin{figure}[t]
    \centering
    \includegraphics[width=\textwidth]{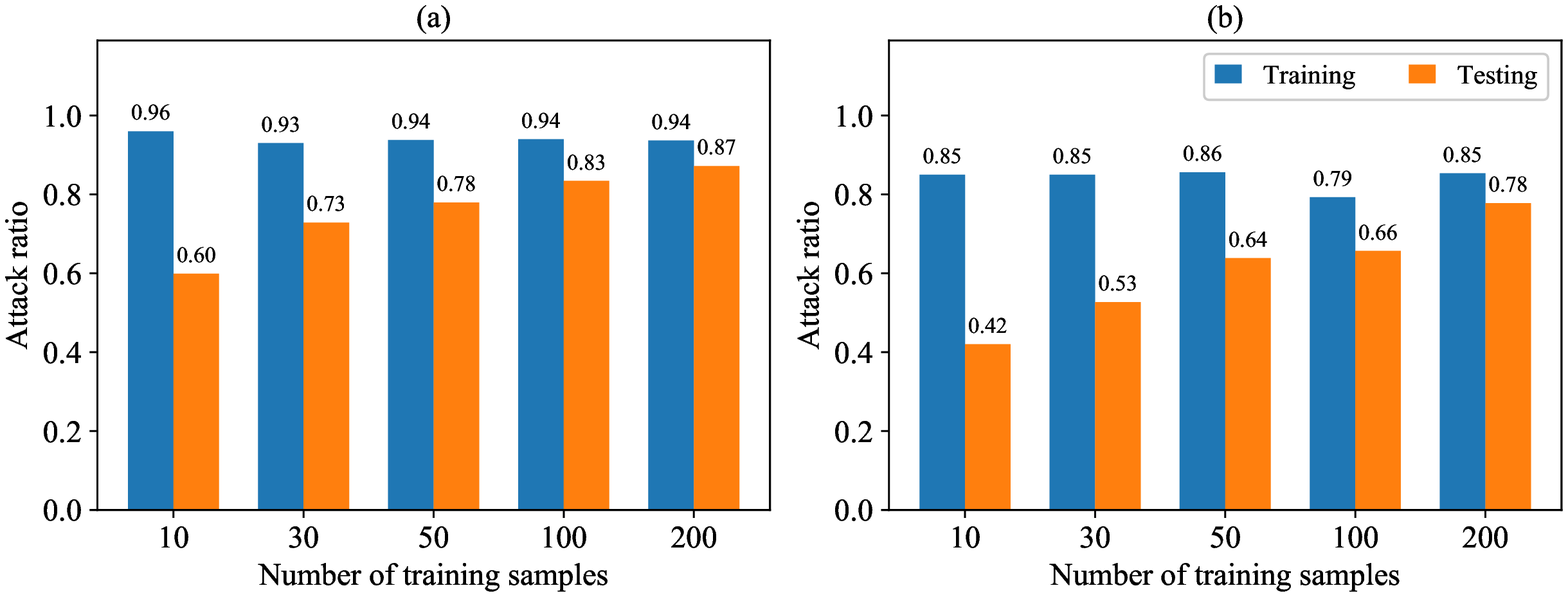}
    \caption{\small Adversarial attack of VGG-19 on   CIFAR-10    with different number of training samples. The $\ell_2$-norm perturbation is fixed as 10.  (a) The results of attacking the dataset from the first class; (b) The results of attacking the whole dataset with 10 classes.}
    \label{fig:cifar_attack_differnt_samples}
\end{figure}

We further show the effects of {\sl different number of training samples} in  Fig.~\ref{fig:cifar_attack_differnt_samples}. 
When  the number of samples is increased, the testing attack ratio is  getting better. 
In our numerical experiments, we set the number of samples as 100 for  one-class experiments and 200 for all-classes experiments.

We  continue to show the {\sl cross-model} performance on four different ResNet networks and one VGG network. 
We test the performance of the attack vector trained from one model on all other models. 
Each row in Table~\ref{tab:cross_model_cifar10} shows the results on the same deep neural network and each column shows the results of the same attack vector. 
It shows that  ResNet-20 is easier to  be attacked compared with  other models. This agrees with our intuition that a simple network structure such as ResNet-20 is less robust.  
On the contrary, VGG-19 is the most robust.   
The success of cross-model attacks   indicates  that these neural networks could find a similar feature.  

\begin{table}[t]
    \centering
    
    \caption{\small Cross-model performance for CIFAR-10}
    \label{tab:cross_model_cifar10}
    
    \begin{tabular}{|c|c|c|c|c|c|}
    \hline
&ResNet-20 & ResNet-44 & ResNet-56 & ResNet-110 & VGG-19 \\
\hline
ResNet-20 &\textbf{91.35\%} & 87.74\% & 86.28\% & 87.38\% & 81.16\% \\
\hline
ResNet-44 &84.75\% & \textbf{92.28\%} & 87.03\% & 85.44\% & 83.44\% \\
\hline
ResNet-56 &83.63\% & 86.67\% & \textbf{90.15\%} & 87.39\% & 84.38\% \\
\hline
ResNet-110 &71.02\% & 77.58\% & 74.19\% & \textbf{92.77\%} & 77.32\% \\
\hline
VGG-19 &53.61\% & 59.74\% & 61.49\% & 66.29\% & \textbf{80.02\%} \\
\hline
    \end{tabular}
    
\end{table}

\subsubsection{CIFAR-100}
Finally, we show the results on CIFAR-100 for both the first class (i.e., dolphin) and all classes.  
Similar to Fashion-MNIST and CIFAR-10,  Fig.~\ref{fig:cifar100_attack} shows that active-subspace can achieve   higher attack ratios than both UAP and  Gaussian random vectors. 
Further, compared with  CIFAR-10,  CIFAR-100 is easier to be attacked  partially because it has more classes.

\begin{figure}[t]
    \centering
    \includegraphics[width=\textwidth]{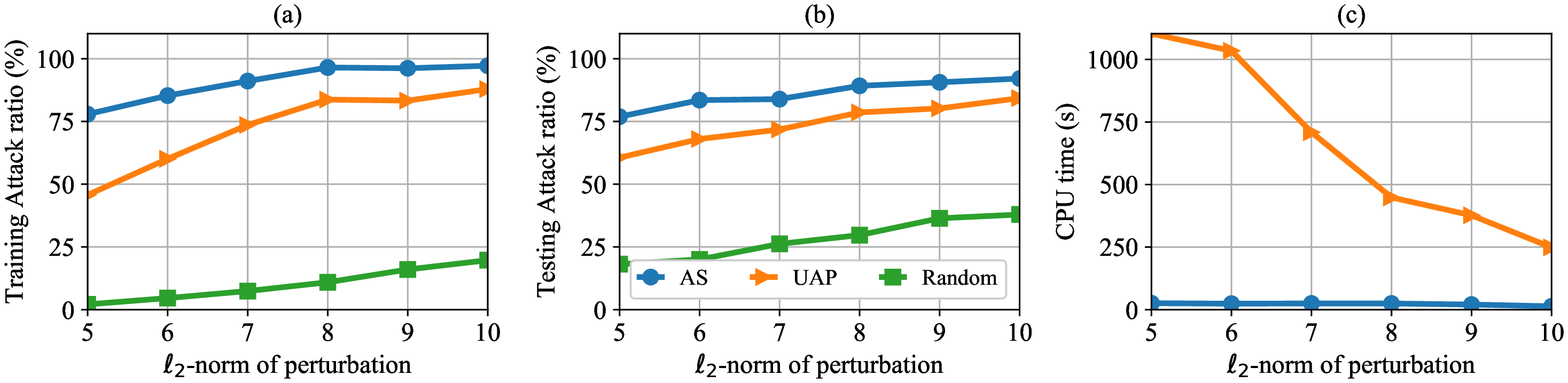}
    \includegraphics[width=\textwidth]{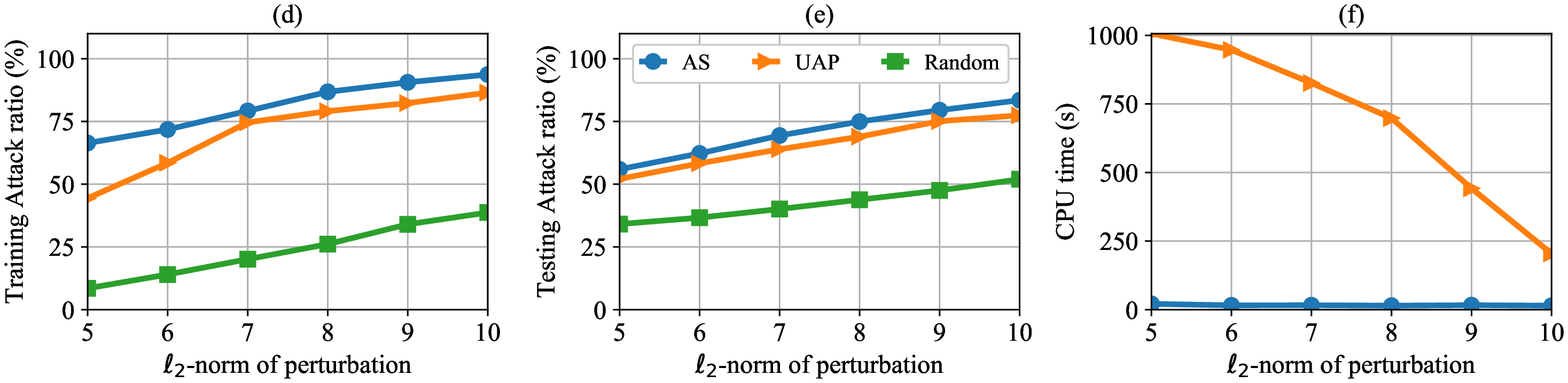}
    \caption{\small Results for universal adversarial attack for  CIFAR-100 with respect to different $\ell_2$-norm perturbations. (a)-(c): The results for  attacking the dataset from the first class. (d)-(f): The results for attacking ten classes dataset together. }
    \label{fig:cifar100_attack}
\end{figure}

We summarize the results for different datasets  in Table~\ref{tab:summary}. 
The second column shows the number of classes in the dataset. 
In terms of testing attack ratio for the whole dataset, active-subspace achieves $24.2\%$, $15\%$, and $6.1\%$ higher attack ratios than UAP for Fashion-MNIST, CIFAR-10, and CIFAR-100,  respectively. 
In terms of the CPU time,  active-subspace achieves $42\times$, $5\times$, and $14\times$ speedup than UAP on the  Fashion-MNIST, CIFAR-10, and CIFAR-100, respectively.

 \begin{table}[t]
     \centering
  \caption{\small Summary of the universal attack for different datasets by the active-subspace compared with UAP and the random vector. The norm of perturbation is equal to   10.}
\label{tab:summary}
     \begin{tabular}{|c|c|ccc|ccc|cc|}
 \hline
 &&\multicolumn{3}{c|}{Training Attack ratio}&\multicolumn{3}{c|}{Testing Attack ratio}&\multicolumn{2}{c|}{CPU time (s)}\\
 \hline
&\# Class  & AS & UAP & Rand & AS & UAP & Rand & AS & UAP\\\hline
Fashion- & 1 &100.0\% & 93.6\% & 1.8\%& \ccf{\textbf 98.0\%} & 91.3\% & 3.0\% &  \ccf{\textbf 0.15} & 5.49 \\
MNIST & 10 &79.2\% & 51.5\% & 8.0\%&  \ccf{\textbf 73.3\%} & 49.1\% & 12.3\% &  \ccf{\textbf 1.40} & 58.85 \\
\hline
\multirow{2}{*}{\small CIFAR-10} 
 & 1 &94.7\% & 79.8\% & 8.0\%& \ccf{\textbf  84.5\%} & 57.9\% & 10.6\% &  \ccf{\textbf 8.18} & 52.83 \\
& 10 &86.5\% & 65.9\% & 10.2\%& \ccf{\textbf 74.9\%} & 59.9\% & 17.0\% &  \ccf{\textbf 37.01} & 181.72 \\
\hline
\multirow{2}{*}{\small CIFAR-100} 
& 1 &97.2\% & 87.9\% & 19.7\%&  \ccf{\textbf 92.1\%} & 84.3\% & 37.9\% &  \ccf{\textbf 13.32} & 248.78 \\
& 100&93.7\% & 86.5\% & 38.7\%&  \ccf{\textbf 83.5\%} & 77.4\% & 52.0\% & \ccf{\textbf 14.32} & 204.50 \\
\hline
     \end{tabular}
 \end{table}

\section{Conclusions \ccf{and Discussions}}
\label{sec:conclusion}

 This paper has   analyzed deep neural networks by the active subspace method originally developed for dimensionality reduction of uncertainty quantification. We have investigated two problems: how many neurons and layers are necessary (or important) in a deep neural network, and how to generate a universal adversarial attack vector that can be applied to a set of testing data?  
 Firstly, we have presented a definition of ``the number of active neurons'' and  have shown its theoretical error bounds for model reduction. Our numerical study has shown that many neurons and layers are not needed. Based on this observation, we have proposed a new network called ASNet by cutting off the whole neural network at a proper layer and replacing all subsequent layers with an active subspace layer and a polynomial chaos expansion layer.
The numerical experiments show that the proposed deep neural network structural analysis method  can produce  a new network with significant storage savings and computational speedup yet with little accuracy loss. 
 Our methods can be combined with existing model compression techniques (e.g., pruning, quantization and low-rank factorization) to develop compact  deep neural network models that are more suitable for the deployment on resource-constrained platforms. Secondly, we have applied the active subspace to generate a universal attack vector that is  independent of a specific data sample and can be applied to a whole dataset. Our proposed method can achieve a much higher attack ratio than the existing work~\cite{moosavi2017universal} and enjoys a lower computational cost. 
 
 \ccf{ASNet has two main goals:   to detect the necessary neurons and layers, and to compress the existing network. 
 To fulfill the first goal, we require a pre-trained model because from Lemmas~\ref{lem: ASerror},  and \ref{lem:err_PCE}, the  accuracy of the reduced model will approach that of the original one.  
 For the second task,  the pre-trained model helps us to get a good estimation for the number of active neurons, a proper layer to cut off, and a good initialization for the active subspace layer and polynomial chaos expansion layer. 
However,  a pre-trained model is not required because we can construct ASNet in a heuristic way (as done in most DNN): 
a reasonable guess for the number of active neurons and cut-off layer, and a random parameter initialization for the pre-model, the active subspace layer and the polynomial chaos expansion layer.} 
 
\ccf{\section*{Acknowledgement}
 We thank the associate editor and referees for their valuable comments and suggestions. 
 } 
 


{\small
\bibliographystyle{siamplain}
\bibliography{ref}
}

\end{document}